\documentclass[10pt]{article}

\usepackage{graphicx}
\usepackage{subfigure} 
\usepackage{amssymb}
\usepackage{pifont}

\usepackage{mathrsfs}

\usepackage{fullpage}
\usepackage[protrusion=true,expansion=true]{microtype}

\usepackage{amsmath}
\usepackage{amsthm}

\usepackage{mathtools}

\usepackage{algorithm}
\usepackage{algorithmic}

\DeclareMathOperator*{\argmin}{arg\,min}

\usepackage{booktabs}       

\usepackage{multirow} 
\usepackage[table,dvipsnames]{xcolor}
\usepackage{colortbl}

\usepackage{times}
\usepackage{natbib}
\usepackage{bm}

\usepackage[colorlinks = true,
            linkcolor = blue,
            urlcolor  = blue,
            citecolor = blue,
            anchorcolor = blue]{hyperref}
\usepackage{cleveref}

\newtheorem{theorem}{Theorem}

\newtheorem{assum}{Assumption}
\newtheorem{lemma}{Lemma}
\newtheorem{proposition}{Proposition}
\newtheorem{remark}{Remark}
\newtheorem{corollary}{Corollary}
\newtheorem{definition}{Definition}

\usepackage[textsize=tiny]{todonotes}

\allowdisplaybreaks

\title{Achieving Linear Speedup in Non-IID Federated Bilevel Learning
\vspace{0.4cm}
} 

\author
{   
Minhui Huang\thanks{Meta; e-mail: {\tt huangmh1995@gmail.com}}
	,~~~Dewei Zhang\thanks{The Ohio State University; e-mail: {\tt   zhangdw.r@gmail.com}} 
	,~~~Kaiyi Ji\thanks{Department of CSE, University at Buffalo; e-mail: {\tt kaiyiji@buffalo.edu}} \footnote{The first two authors contributed equally. The corresponding author: Kaiyi Ji.}
}

\begin{document}

\maketitle

\begin{abstract}
Federated bilevel optimization has received increasing attention in various emerging machine learning and communication applications. Recently, several Hessian-vector-based algorithms have been proposed to solve the federated bilevel optimization problem. However, several important properties in federated learning such as the partial client participation and the linear speedup for convergence (i.e., the convergence rate and complexity are improved linearly with respect to the number of sampled clients) in the presence of non-i.i.d.~datasets, still remain open. In this paper, we fill these gaps by proposing a new federated bilevel algorithm named FedMBO with a novel client sampling scheme in the federated hypergradient estimation. We show that FedMBO achieves a convergence rate of $\mathcal{O}\big(\frac{1}{\sqrt{nK}}+\frac{1}{K}+\frac{\sqrt{n}}{K^{3/2}}\big)$ on non-i.i.d.~datasets, where $n$ is the number of participating clients in each round, and $K$ is the total number of iteration. This is the first theoretical linear speedup result for non-i.i.d.~federated bilevel optimization.     
Extensive experiments validate our theoretical results and demonstrate the effectiveness of our proposed method. 
\end{abstract}

\section{Introduction}
\label{introduction}

Federated learning is a privacy-preserving training paradigm over distributed networks that are designed for edge computing~\citep{mcmahan2017communication}. In federated learning, multiple edge devices (or clients) work together to learn a global model under the coordination of a central server. Instead of transmitting user data directly to the central server, each client stores data and computes locally and only transmits the privacy-preserving information. This paradigm is increasingly attractive due to the growing computational power of edge devices and the increasing demand for privacy protection. 
Federated learning is facing more challenges than traditional distributed optimization due to the high communication cost, data and system heterogeneity, and privacy concerns. Recent years have witnessed great progress in the algorithmic design and system deployment to address such challenges~\citep{wang2021cooperative,karimireddy2019scaffold,stich2020error}. 

Recently, federated bilevel learning has received increasing attention~\citep{chen2018federated,fallah2020personalized,zeng2021improving} because many modern machine learning problems naturally exhibit a bilevel optimization structure. For example, \citealt{chen2018federated, fallah2020personalized} studied the federated meta-learning problems, \citealt{khodak2021federated} proposed federated hyperparameter optimization approaches, and \citealt{zeng2021improving} improved the fairness in federated learning using a bilevel method. This motivates us to study the following federated bilevel optimization problem. 
\begin{align}\label{objective}
&\min_{x\in\mathbb{R}^{p}} \Phi(x)=f(x, y^*(x)) : = 
\frac{1}{m} \sum_{i=1}^m f_i(x, y^*(x)) \nonumber
\\& \;\;\mbox{s.t.} \quad y^*(x)\in \argmin_{y\in\mathbb{R}^q} g(x,y):=
\frac{1}{m}\sum_{i=1}^m g_i(x,y),
\end{align}
where $f_i(x, y) = \mathbb{E}f_i(x, y;\xi^i),\, g_i(x, y) = \mathbb{E}g_i(x, y;\zeta^i)$ are stochastic upper- and lower-level loss functions of client $i$, and $m$ is the total number of clients. Existing federated learning algorithms like FedAvg and its variants \citep{mcmahan2017communication} cannot be applied to solve the federated bilevel problem \cref{objective} due to the nested optimization structure,  the global Hessian inverse estimation in the hypergradient (i.e., $\nabla\Phi(x)$) computation, and the data heterogeneity in both the upper- and lower-level problems.

\begin{table*}[t]
\caption{Comparison of FedMBO with existing federated bilevel algorithms. $m$ is the total number of clients, $n$ is the size of sampled clients, and $\epsilon$ is the required accuracy. }
\label{sample-table}
\vskip 0.1in
\begin{center}
\begin{small}
\begin{tabular}{lcccc}
\toprule
Algorithm  & Sample Complexity& Partial Client Participation & Linear Speedup & Data Heterogeneity  \\
\midrule
LocalBSGVR \citep{gao2022convergence}    &$\mathcal{O}(\epsilon^{-3/2} m^{-1}$) &\ding{55}  &\ding{51}  &\ding{55} \\
AdaFBiO \citep{huang2022fast}    &$\mathcal{O}(\epsilon^{-3/2}$) &\ding{55}  &\ding{55}  &\ding{51} \\
FedNest \citep{tarzanagh2022fednest}   &$\mathcal{O}(\epsilon^{-2}$) &\ding{55} &\ding{55} &\ding{51}\\
FedMBO     &$\mathcal{O}(\epsilon^{-2} n^{-1}$) &\ding{51}  &\ding{51}  &\ding{51} \\
\bottomrule
\end{tabular}
\end{small}
\end{center}
\end{table*}

Recently, several approaches~\citep{li2022local,tarzanagh2022fednest,gao2022convergence, huang2022fast} have been proposed to efficiently solve \cref{objective}. \citealt{li2022local} considered a special case of \cref{objective}, where the lower-level problem is minimized only locally, i.e., $y_i^*(x)=\argmin_y g_i(x,y)$ for each client $i$. For the general case, 
\citealt{gao2022convergence} focused on the homogeneous setting with i.i.d.~datesets and proposed momentum-based distributed bilevel algorithms. In the more practical but challenging heterogeneous setting with non-i.i.d.~datasets, \citealt{huang2022fast} proposed a momentum-based method AdaFBiO based on fully local hypergradient estimators. 
\citealt{tarzanagh2022fednest} proposed FedNest based on an implicit differentiation based federated hypergradient estimator. In the inner loop, FedNest calls $T$ times of FedInn, which is a federated stochastic variance reduced gradient (FedSVRG) algorithm, to solve the lower-level problem. Then FedNest calls FedOut, which constructs a federated hypergradient estimator, to optimize the upper-level problem. 
 However, as shown in \Cref{sample-table}, both AdaFBiO and FedNest fail to achieve a linear speedup for convergence in training due to the fully local hypergradient estimation, and 
 the high correlation among the individual hypergradient estimators computed by all clients, respectively. In addition, they are restricted to the full client participation.  
 Then, an important but open question remains: 
\begin{list}{$\bullet$}{\topsep=1ex \leftmargin=0.2in \rightmargin=0.in \itemsep =0.06in}
\item {\em Can we develop an easy-to-implement federated method, which achieves a linear speedup for convergence in the general heterogeneous setting, and allows flexible partial client participation?}
\end{list}

\vspace{0.2cm}

\noindent\textbf{Our contributions.} 
In this paper, 
we provide an affirmative answer to the above question by proposing a novel federated algorithm called Federated Minibatch Bilevel Optimization (FedMBO). Our contributions are summarized as follows.
\begin{list}{$\bullet$}{\topsep=0.2ex \leftmargin=0.2in \rightmargin=0.in \itemsep =0.06in}
    \item The proposed FedMBO follows a double-loop scheme in bilevel optimization and consists of two important components. For the inner loop,  FedMBO adopts a simple Minibatch Stochastic Gradient Descent (SGD) algorithm. Compared with FedAvg and FedSVRG, the minibatch SGD and its accelerated variant are more immune to the heterogeneity of the problem \citep{woodworth2020minibatch}, which is critical in achieving  the linear speedup for convergence under the bilevel optimization structure.  For the outer loop, FedMBO features a Parallel Hypergradient Estimator (PHE) with a novel multi-round client sampling scheme. Compared to IHGP \citep{tarzanagh2022fednest}, our PHE procedure allows either full or partial client participation, and more importantly, achieves a variance bound linearly decreasing w.r.t.~the number of participating clients. We anticipate that PHE can be of independent interest to other settings such as decentralized or asynchronous bilevel optimization. 

    \item We show that FedMBO achieves a convergence rate of $\mathcal{O}\big(\frac{1}{\sqrt{nK}}+\frac{1}{K}+\frac{\sqrt{n}}{K^{3/2}}\big)$ and a sample complexity (i.e., the number of samples to achieve an $\epsilon$-stationary point) of $\mathcal{O}(\epsilon^{-2}n^{-1})$, which outperforms that of FedNest~\citep{tarzanagh2022fednest} by an order of $n$ due to the linear speedup. As shown in \Cref{sample-table}, compared to the momentum-based LocalBSGVR~\citep{gao2022convergence} and AdaFBiO~\citep{huang2022fast}, our FedMBO is more flexible with partial client participation, and more importantly, achieves the linear speedup for convergence even in the presence of data heterogeneity.   
    
    \item We conduct extensive experiments to validate our theoretical results, and further demonstrate the effectiveness of our proposed federated hypergradient estimator and the FedMBO algorithm.    
\end{list}

\subsection{Related Work}

{\bf Bilevel optimization approaches}: Bilevel optimization was first introduced in 1970's \citep{bracken1973mathematical} and has been being studied in the past decades. Since then, tremendous efforts have been made to reformulate the bilevel problem as a single-level optimization problem and develop efficient algorithms to solve it\citep{aiyoshi1984solution, edmunds1991algorithms, hansen1992new, shi2005extended}. Recently, several prevailing machine learning applications can be naturally formulated as a bilevel programming problem \citep{maclaurin2015gradient, pedregosa2016hyperparameter,finn2017model, franceschi2017forward, franceschi2018bilevel, ji2020convergence}, which brings a lot of attention to the bilevel programming in the machine learning community. On the theoretical side, there are many existing works deriving both asymptotic \citep{franceschi2018bilevel, shaban2019truncated,liu2021value} and non-asymptotic \citep{ghadimi2018approximation, ji2021bilevel, hong2020two, chen2021closing,guo2021randomized,huang2022efficiently} convergence analysis for the determinstic or stochastic bilevel optimization. For example, \citealt{ghadimi2018approximation,hong2020two,ji2021bilevel,arbel2022amortized} proved the convergence for SGD type of bilevel methods via the approximate implicit differentiation (AID) approach. 
\citealt{yang2021provably, chen2021single, khanduri2021near,guo2021randomized,dagreou2022framework} adopted the variance reduction and momentum techniques into stochastic bilevel programming to achieve better complexity results. 

\vspace{0.2cm}
\noindent {\bf Federated learning}: At the core of federated learning is the prevailing FedAvg algorithm and its variants \citep{mcmahan2017communication, li2020federated, karimireddy2019scaffold, mitra2021linear, acar2021federated, stich2018local, yu2019parallel, yang2020achieving, qu2020federated} to address the communication efficiency and the data privacy concerns. We review literature with a focus on the analysis of the linear speedup for convergence. In the homogeneous setting, two variants of FedAvg were proposed to achieve linear speedup \citep{stich2018local, yu2019parallel} under the assumptions of bounded gradient and full client participation. Later, \citealt{wang2021cooperative, stich2020error} removed the bounded gradient assumption and established a convergence rate of $\mathcal{O}(\epsilon^{-2}m^{-1})$. In the heterogeneous setting, the SCAFFOLD algorithm \citep{karimireddy2019scaffold} achieves the first linear speedup convergence rate using a variance reduction framework and is independent of the level of heterogeneity. After that, several variants of FedAvg \citep{yang2020achieving, qu2020federated} have also been proved to achieve linear speedup. Another interesting line of work focuses on the comparison between FedAvg and minibatch SGD \citep{woodworth2020local, woodworth2020minibatch}. In the homogeneous case, FedAvg provably outperforms minibatch SGD and its accelerated versions~\citep{woodworth2020local}. However, when the heterogeneity level is high, FedAvg is shown to be worse than minibatch SGD. 

\vspace{0.2cm}
\noindent {\bf Distributed bilevel optimization}:
For the decentralized stochastic bilevel optimization (DSBO) problem, \citealt{lu2022decentralized, terashita2022personalized} studied the setting where the clients have their own local lower problems and thus the communication for the lower-level part can be saved, and  \citealt{yangdecentralized, chen2022decentralized, chen2022decentralized2} considered a more general global setup, in which all the clients target solving a global lower-level problem together. The most related works to this paper is the FedNest algorithm \citep{tarzanagh2022fednest}, which achieves a sample complexity of $\mathcal{O}(\epsilon^{-2})$. This result was further improved by the momentum-based federated bilevel algorithms in~\citealt{gao2022convergence,li2022local} in the homogeneous setting.
Our proposed FedMBO achieves the first linear speedup result in the heterogeneous setting. 


\section{Definitions and Assumptions}
Throughout this paper, we make the following standard assumptions, as typically adopted in bilevel optimization. 
\begin{definition}
A function $h:\mathbb{R}^{n_1}\mapsto\mathbb{R}^{n_2\times n_3}$ is Lipschitz continuous with constant $L$ if 
\begin{equation}
    \left\|h(z_1)-h(z_2)\right\|\leq L\left\|z_1-z_2\right\| \ \ \forall z_1, z_2\in\mathbb{R}^{n_1},\nonumber
\end{equation}
where $\left\|\cdot\right\|$ denotes the Euclidean norm of a vector or matrix depending on the value of $n_3$.
\end{definition}
\begin{definition}
    A solution $x$ is $\epsilon$-accurate stationary point if $\mathbb{E}\left\|\nabla \Phi(x)\right\|^2\leq \epsilon$, where $x$ is the output of an algorithm. 
\end{definition}

 Let $z=(x,y)\in\mathbb{R}^{p+q}$ denotes all parameters. 

\begin{assum}\label{ass:lipschitz}(Lipschitz properties).
For all $i\in [m]:$
$f_i(z)$, $\nabla f_i(z)$, $\nabla g_i(z)$, $\nabla^2 g_i(z)$ are $\ell_{f,0}$, $\ell_{f,1}$, $\ell_{g,1}$, $\ell_{g,2}$-Lipshitz continuous, respectively.
\end{assum}

\begin{assum}\label{ass:strong_convex}(Strong convexity)
For all $i\in [m]:$
$g_i(x,y)$ is $\mu_g$-strongly convex in $y$ for any fixed $x\in\mathbb{R}^{q}$.
\end{assum}

\begin{assum}\label{ass:sto_sample}(Unbiased estimators).
For all $i\in [m]:$
$\nabla f_i(z;\xi)$, $\nabla g_i(z;\zeta)$, $\nabla^2 g_i(z;\zeta)$ are unbiased estimators of $\nabla f_i(z)$, $\nabla g_i(z)$, $\nabla^2 g_i(z)$, respectively. 
\end{assum}

\begin{assum}\label{ass:bounded_var}(Bounded variances).
For all $i\in [m]:$ there exist constants $\sigma_f^2$, $\sigma^2_{g,1}$, and $\sigma^2_{g,2}$, such that 
\begin{align}
    &\mathbb{E}_{\xi}\left\|\nabla f_i(z;\xi)-\nabla f_i(z)\right\|^2]\leq \sigma_f^2,
    \nonumber
    \\
    &\mathbb{E}_{\zeta}\left\|\nabla g_i(z;\zeta)-\nabla g_i(z)\right\|^2\leq \sigma_{g,1}^2,
    \nonumber
    \\
    &\mathbb{E}_{\zeta}\left\|\nabla^2 g_i(z;\zeta)-\nabla^2 g_i(z)\right\|^2\leq \sigma_{g,2}^2.
    \nonumber
\end{align}
\end{assum}

As typically adopted in the analysis for partial client participation in federated learning, the following assumption controls the difference between the local gradient at each client $\nabla g_i(z)$ and the global gradient $\nabla g(z)$. 
\begin{assum}\label{ass:bounded_global_variance} For all $i\in [m]:$ there exist a constant $\sigma_g$, such that $\mathbb{E}\left\|\nabla g_i(z)-\nabla g(z)\right\|^2\leq \sigma_g^2.$
\end{assum}

\section{Algorithms}\label{sec:alg}
To solve the bilevel problem in \cref{objective}, the biggest challenge lies in computing the federated hypergradient $\nabla \Phi(x)=(1/m)\sum_{i=1}^m \nabla f_i(x,y^*(x))$, whose explicit form can be obtained as follows via implicit differentiation. 
\begin{lemma}\label{le:implicitfunction}
    Under Assumptions~\ref{ass:lipschitz} and~\ref{ass:strong_convex}, we have 
    \begin{align}
        \nabla f(x,y^*&(x)) = \nabla_x f(x,y^*(x))-\nabla^2_{xy} g(x,y^*(x))\left[\nabla^2_{yy} g(x,y^*(x))\right]^{-1}\nabla_y f(x,y^*(x)),\label{lemma:implicitfunction}
    \end{align}
    where $\nabla^2_{yy} g(x,y)$ is defined as the Hessian matrix of $g$ with respect to $y$ and $\nabla^2_{xy} g(x,y)$ is defined as
    \begin{equation}
    \nabla^2_{xy} g(x,y) := \begin{bmatrix} 
    \frac{\partial^2}{\partial x_1 \partial y_1}g(x,y) & \ldots &\frac{\partial^2}{\partial x_1 \partial y_q}g(x,y) 
    \\
    & \ldots & 
    \\
    \frac{\partial^2}{\partial x_p \partial y_1}g(x,y) & \ldots &\frac{\partial^2}{\partial x_p \partial y_q}g(x,y)
    \end{bmatrix}.
    \nonumber
\end{equation}
\end{lemma}
To employ the above lemma, several challenges arise. First, the evaluation of the federated hypergradient in~\cref{lemma:implicitfunction} requires the approximation of the minimizer $y^*(x)$ of the lower-level problem, which may introduce a big bias due to the client drift. 
We propose to use the simple minibatch SGD as the lower-level optimizer, as elaborated in~\Cref{algoelaborate:mini_bacth}, to mitigate the impact of the lower-level client drift on the final convergence rate. Second, the stochastic approximation of the infeasible Hessian inverse matrix $\left[\nabla^2_{yy} g(x,y^*(x))\right]^{-1}\nabla_y f(x,y^*(x))$  in \Cref{le:implicitfunction} often involves the computation of a series of global Hessian-vector products in a nonlinear manner, which complicates the implementation and may introduce a large estimation variance.  
Third, the federated hypergradient estimation may suffer from a large bias 
due to both the upper- and lower-level client drifts. 
In this paper, we propose a new algorithm FedMBO, which contains two main components, i.e., a minibatch SGD based lower-level optimizer and a novel federated hypergradient estimator, to address the above challenges, respectively.   
\begin{algorithm}[h]
  \caption{{Heterogeneous Distributed Minibatch Bilevel Optimization with Partial Clients Participation}}
  \label{alg:HDMB-Partial} 
  \begin{algorithmic}[1]
    \STATE {\bfseries Input:} full client index set $[m]$, partial clients $n$ participation, batch size $S$ of local SGD at inner loop, initial point $(x^0, y^0)$, $N\in\mathbb{N}^+$
    \FOR{$k = 0, 1, \ldots K-1$}
    \STATE  $y^{k,0} = y^k$
    \FOR{$t = 0, 1, \ldots T-1$}
    \STATE Sample client index subset $C^{k,t}=\big\{c^{k,t}_1,...,c^{k,t}_n\big\}$ with $\left |C^{k,t} \right | = n$ 
    \FOR{$i \in \left[n\right]$ in parallel}
    \STATE Sample batch $S_{i}^{k, t} = \{ \xi_{i, 0}^{k, t}, \xi_{i, 1}^{k, t}, \ldots, \xi_{i, S-1}^{k, t} \}$ 
    \STATE Compute {\small$G_i^{k, t} = \frac{1}{S} \sum_{j = 0}^{S-1} \nabla g_{c^{k,t}_i}(x^k,y^{k, t}; \xi_{i, j}^{k, t})$}
    \ENDFOR
    \STATE $G^{k, t} =  \frac{1}{n} \sum_{i = 1}^{n}G_i^{k, t}$
    \STATE $y^{k,t+1} = y^{k,t} - \beta_{k,t} G^{k, t}$
    \ENDFOR
    \STATE $y^{k+1} = y^{k, T}$
    \STATE $\{\mathcal{H}_i\} = {\textbf{PHE}}(x^k, y^{k+1},N,n)$
    \STATE $h = \frac{1}{n} \sum_{i = 1}^{n} \mathcal{H}_i$
    \STATE $x^{k+1} = x^k - \alpha_k h$
    \ENDFOR
  \end{algorithmic}
\end{algorithm}

\subsection{Minibatch SGD for Lower-level Updates}\label{algoelaborate:mini_bacth}
 To efficiently solve the lower-level problem, one popular approach is FedAvg. Starting from a common initialization, the clients in FedAvg run multiple local SGD updates on its own objective, which are then  aggregated to update the inner variable $y$. However, it has been shown in \citealt{tarzanagh2022fednest} that FedAvg introduces an undesirable hypergradient estimation bias due to the large client drift. Thus, they proposed FedLin, as a variant of the variance reduction method FedSVRG~\citep{mitra2021linear}, to mitigate the impact of the client drift. However, FedLin has a more complex implementation due to the nest SVRG loop, and more importantly, 
as shown in \citealt{tarzanagh2022fednest}, 
its convergence error induced by the client drift 
is not linearly decreasing w.r.t.~the number of sampled clients, which is one crucial factor in missing the linear speedup in the convergence rate.

Inspired by a recent work \citep{woodworth2020minibatch}, we use the minibatch SGD as the lower-level solver, where
the clients compute their local minibatch stochastic gradients, which are further aggregated for a one-step update on $y$. In specific, 
we first sample a subset $C^{k,t}=\big\{c^{k,t}_1,...,c^{k,t}_n\big\}$ of clients, and each of them draws a local data batch $S_{i}^{k, t} = \{ \xi_{i, 0}^{k, t}, \xi_{i, 1}^{k, t}, \ldots, \xi_{i, S-1}^{k, t} \}$ with $\big |S_i^{k, t} \big| = S$ and computes the local stochastic gradient $\nabla g_{c^{k,t}_i}(x^k,y^{k, t}; \xi_{i, j}^{k, t})$. Then, the server aggregates the gradients as $$G^{k, t} = \frac{1}{nS} \sum_{i = 1}^{n} \sum_{j = 0}^{S-1} \nabla g_{c^{k,t}_i}(x^k,y^{k, t}; \xi_{i, j}^{k, t}),$$ and further run one-step SGD to update $y^{k,t}$ as  
$$y^{k,t+1} = y^{k,t} - \beta_{k,t} G^{k, t}.$$ 
 Compared with FedAvg and FedLin, the minibatch SGD admits a simpler implementation, and more importantly, is more resilient  to the data heterogeneity by  a more aggressive single update at all clients. As will be seen later, minibatch SGD provides a more accurate estimation of the lower-level solution, which is necessary in achieving the linear speedup. 
\begin{remark}
In the minibatch SGD implementation, we set the batch size to be larger than FedAvg, and hence more aggressive per-iteration progress is made. Thus, the computational cost of minibatch SGD is comparable to FedAvg. More importantly, minibatch SGD admits a much smaller client drift, which is critical in achieving the linear speedup. 
\end{remark}
\begin{remark}
    In the experiments (see \Cref{exp}), we demonstrate the great advantages of minibatch SGD over FedAvg in mitigating the client drift during the bilevel training, and in improving the overall communication efficiency. 
\end{remark}

\subsection{Federated Hypergradient Procedure}\label{algoelaborate:hypergradient}
In the non-federated setting, one often defines the surrogate
\begin{align}
\overline{\nabla} f(x,y) &= \nabla_x f(x,y)
- \nabla^2_{xy} g(x,y)[\nabla^2_{yy} g(x,y)]^{-1}\nabla_y f(x,y)\label{surrogate_bar_f}
\end{align}
to efficiently approximate the hypergradient $\nabla f(x,y^*(x))$ in~\cref{lemma:implicitfunction}. 
Compared with~\cref{lemma:implicitfunction}, the surrogate simply replaces $y^*(x)$ by its approximation $y$. A typical approach for efficiently approximating the surrogate is to use the Neumann series-based stochastic estimator.
\begin{align}\label{stoapprox}
    \overline{\nabla} f(x,y) \approx &\nabla_x f(x,y;\xi)-\nabla^2_{xy} g\big(x,y,;\zeta^{N^{'}+1}\big)\bigg[\frac{N}{l_{g,1}}\prod_{n=1}^{N^{'}}\Big(I-\frac{1}{l_{g,1}}\nabla_{yy}^2 g(x,y;\zeta^{n})\Big)\bigg]\nabla_y f(x,y;\xi)
\end{align}
where $N^{'}$ is  chosen from $\{0,...,N-1\}$ uniformly at random and $\{\xi,\zeta^1,...,\zeta^{N^{'}+1}\}$ are i.i.d.~samples. Particularly,~\citealt{ghadimi2018approximation,hong2020two} show that the inverse Hessian estimation bias exponentially decreases with the number of samples $N$, i.e., 
\begin{align}
 \Big\|\Big[&\nabla^2_{yy} g(x,y)\Big]^{-1}-\mathbb{E}\Big[\frac{N}{\ell_{g,1}}\prod_{n=1}^{N^{'}}\Big(I-\frac{1}{\ell_{g,1}}\nabla_{yy}^2 g(x,y;\zeta^{n})\Big)\Big]\Big\|\leq \frac{1}{\mu_g}\Big(1-\frac{\mu_g}{\ell_g}\Big)^N\label{Neumann_series}
\end{align}
where the expectation is taken with respect to both $N'$ and $\zeta$. 
However, in the federated setting, the computation of the hypergradient is challenging due to client drift by the data heterogeneity, and the computation of a   
series of global Hessian matrices in a nonlinear manner, as shown in \cref{stoapprox}. 
To address such challenges,~\citealt{tarzanagh2022fednest} proposed 
the following federated hypergradient estimator:
\begin{equation}
    h_i:= \nabla_x f_i(x,y;\xi) -\nabla^2_{xy} g_i(x,y;\zeta^{N^{'}+1})p_{N^{'}},
    \nonumber
\end{equation}
where the global estimator $p_{N'}$ of  the Hessian-inverse-vector product $[\nabla^2_{yy} g(x,y)]^{-1}\nabla_y f(x,y)$ takes the form of 
\begin{align*}
    p_{N'} = \frac{N}{\ell_{g,1}}\prod_{n=1}^{N'}\Big(I - \frac{1}{\ell_{g,1}\left|S_n\right|}\sum_{i=1}^{\left|S_n\right|}\nabla_{yy}^2 g_i(x,y;\zeta_{i,n})\Big)\frac{1}{\left|S_0\right|} \sum_{i\in S_0} \nabla_y f_i(x,y;\xi_{i,0}),
    \nonumber
\end{align*}
which is constructed by computing and aggregating local Hessian-vector products in $N'$ communication rounds.  

\begin{algorithm}[h]
    \caption{{Parallel Hypergradient Estimator with $n$ Clients Participation (PHE)}}
   \label{alg:PHE_n}
 \begin{algorithmic}[1]
   \STATE Sample clients $C^0 = \{c^0_1,...,c^0_n\}$ 
   \FOR{$i \in [n]$ in parallel}
   \STATE Sample data points $\theta^{i}, \phi^i$
   \STATE Compute $d_i=\nabla_y f_{c_i^0}(x^k, y^{k+1}; \phi^{i})$
   \STATE Compute $p_{i, 0} = \frac{N}{l_{g,1}}\nabla_y f_{c_i^0}(x^k, y^{k+1}; \theta^{i})$
   \STATE Generate $N_i \in\left\{0,1,...,N-1\right\}$ uniform at random
   \ENDFOR
   \FOR{$l = 1, \ldots, \max\left\{N_i,i\in\left[n\right]\right\} $}
   \STATE Sample $C^l = \{c^l_1,...,c^l_n\}$ 
   \FOR{$i\in [n]$ in parallel}
   \STATE Sample a data point $\zeta^{i, l}$
   \IF{$l \le N_i$}
   \STATE Compute $p_{i,l}=\big(I - \frac{1}{\ell_{g, 1}}\nabla_{yy}^2 g_{c_i^l}(x^k, y^{k+1};\zeta^{i, l})\big)p_{i,l-1}$
   \ELSE
   \STATE Set $p_{i,l}=p_{i,l-1}$
   \ENDIF
   \ENDFOR
   \ENDFOR
   \STATE Sample  $C^{\max \left\{N_i\right\}+1} = \{c^{\max \left\{N_i\right\}+1}_1,...,c^{\max \left\{N_i\right\}+1}_n\}$ 
   \FOR{$i\in [n]$ in parallel}
   \STATE Sample a data point $\omega^i$
   \STATE $\mathcal{H}_i=d_i - \nabla^2_{xy} g_{{c_i^{\max \left\{N_i\right\}+1}}}(x^k, y^{k+1}; \omega^i) p_{i,\max \left\{N_i\right\}}$
   \ENDFOR
   \STATE Return $\mathcal{H}=\{\mathcal{H}_i\}_{i\in [n]}$
 \end{algorithmic}
 \end{algorithm}

 \begin{figure*}[t]
\centering
     \includegraphics[scale=0.55]{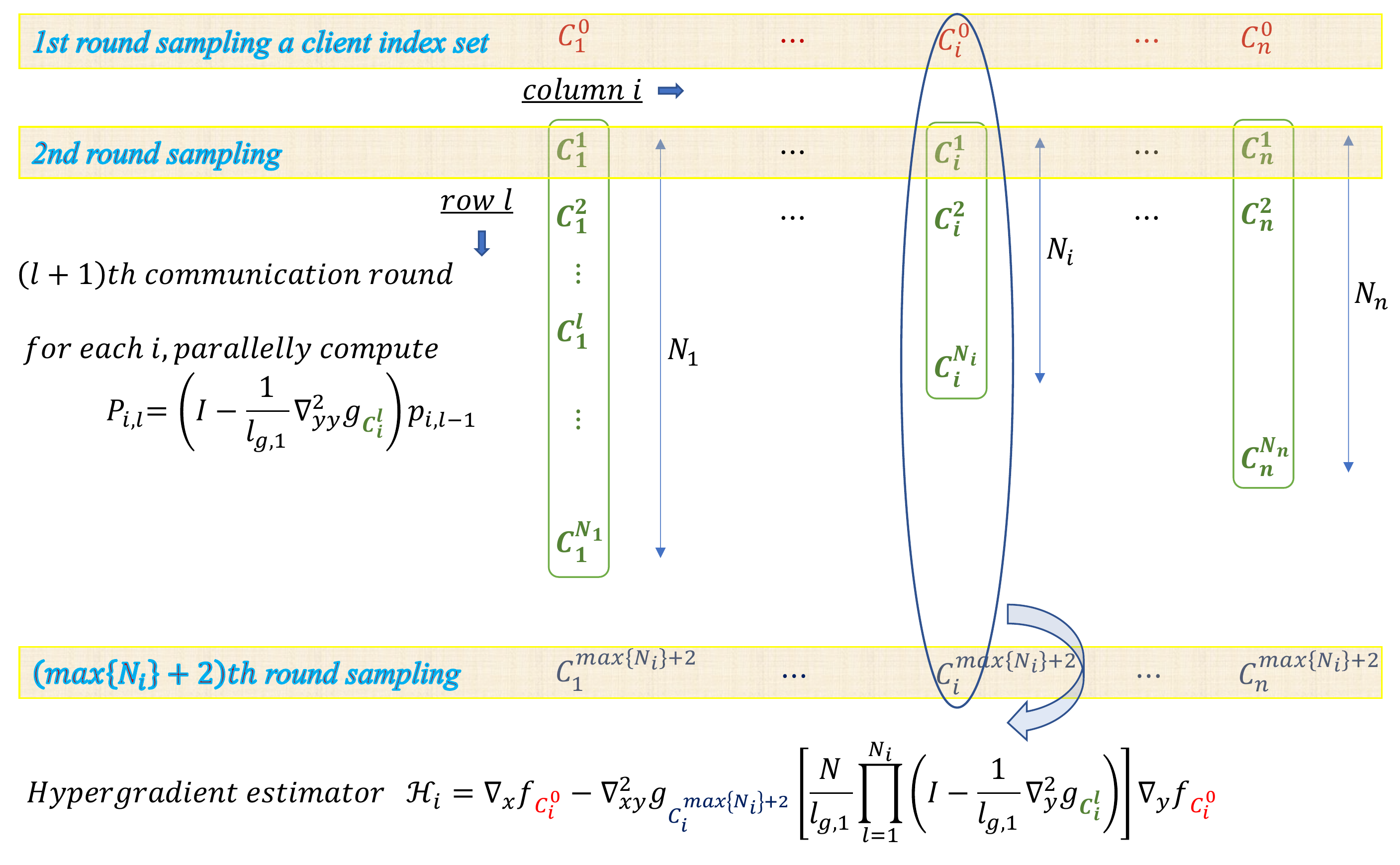}
    \caption{Illustration diagram of client sampling for the federated hypergradient estimation in~\Cref{alg:PHE_n}.}\label{fig:samplingdiagram}
\end{figure*}

However, there are two main limitations of the above federated hypergradient estimator. First, the estimator requires full client participation because each client $i$ needs to compute an $h_i$. 
Second, the $h_i's$ are highly correlated due to the shared global estimation $p_{N^{'}}$. As a result, the variance of $\frac{1}{m}\sum_{i=1}^m h_i$ cannot be shown to decay w.r.t.~$m$, which turns out to be the bottleneck for achieving the linear speedup.


To deal with these challenges, we propose a new federated hypergradient estimator with a novel client sampling and communication scheme. 
{As shown by~\Cref{alg:PHE_n} and illustrated by \Cref{fig:samplingdiagram}, each communication round $l$ (highlighted by the yellow shallow in \Cref{fig:samplingdiagram})  samples $n$ clients ($n\leq m$) indexed by $\{c_0^l,...,c_n^l\}$, and then the sampled clients compute the Hessian-vector product $(I - \frac{1}{\ell_{g, 1}}\nabla_{yy}^2 g_{c_i^l}(x^k, y^{k+1};\zeta^{i, l}))p_{i,l-1}$, which are used for the Hessian-vector construction in the next communication round. In the vertical direction of \Cref{fig:samplingdiagram} (i.e., from line 8 to line 18 in~\Cref{alg:PHE_n}),}
the clients in each column are involved to construct an individual component $\mathcal{H}_i$  of the federated hypergradient estimator.
The proposed estimators $\mathcal\{H_i\}$ take the form  of 
\begin{align}\label{def_of_hypergrad}
    \mathcal{H}_i(x^k,y^{k+1})
     = &\nabla_x f_{c^0_i}(x^k, y^{k+1}; \phi^{i})- \nabla^2_{xy} g_{c_i^{\max\{N_i\}+1}}(x^k, y^{k+1}; \omega^i)
    \nonumber
    \\
    &\quad\times\bigg[\frac{N}{l_{g,1}}\prod_{l=1}^{N_i}\Big(I - \frac{1}{\ell_{g, 1}}\nabla_{yy}^2 g_{C_i^l}(x^k, y^{k+1};\zeta^{i, l})\Big)\bigg]\nabla_y f_{c^0_i}(x^k, y^{k+1}; \theta^i). 
\end{align}


\subsection{Entire Procedure}\label{algoelaborate:wrapup}
The previous two sections describe the lower-level updating procedure on $y$ and the federated hypergradient estimator of the proposed FedMBO method. In this section, 
we briefly summarize the whole algorithm, which is formally described in~\Cref{alg:HDMB-Partial}. At the beginning of FedMBO, we specify the number of participating clients $n\leq m$, the batch size $S$ for the minibatch SGD implemented at the inner loop, and the constant $N$ controlling the Hessian inverse approximation accuracy. At each round $k=0,1,..., K-1$, FedMBO first runs minibatch SGD to update $y$, then constructs the federated hypergraident estimator using \Cref{alg:PHE_n}, and finally updates the outer variable $x$ based on the hypergradient  estimator. We do not run multiple local updates in the updates of $x$ because the federated hypergradient estimator $h$ requires the global information, which is unavailable for local updates of each client.  


\section{Main Results}\label{mainresults}
As discussed in the previous section,~\Cref{alg:PHE_n} generates  the federated hypergradient estimators $\left\{\mathcal{H}_i\right\}$ for estimating $\nabla f(x,y^*(x))$. With slight abuse of notation, we define $\mathcal{H}_i(x^k,y^{k+1})$ to be the output of~\Cref{alg:PHE_n} at the $k$-th round of~\Cref{alg:HDMB-Partial}. For different $i,j$, we have $\mathbb{E}\left[\mathcal{H}_i(x^k,y^{k+1})\right]=\mathbb{E}\left[\mathcal{H}_j(x^k,y^{k+1})\right]$ 
and $$\mathbb{E}\left[\mathcal{H}_i(x^k,y^{k+1})|\mathcal{F}^k\right]=\mathbb{E}\left[\mathcal{H}_j(x^k,y^{k+1})|\mathcal{F}^k\right],$$ where $\mathcal{F}^k:=\sigma\left\{y^0,x^0,...,y^k,x^k,y^{k+1}\right\}$ denotes the filtration that captures all the randomness up to the $k$-th outer loop. We denote $\overline{\mathcal{H}}(x,y):=\mathbb{E}\left[\mathcal{H}_i(x^k,y^{k+1})|\mathcal{F}^k\right]$. Referring to~\Cref{{algoelaborate:hypergradient}}, the resulted $\overline{\mathcal{H}}(x,y)$ is ``close'' to the surrogate function $\overline{\nabla}f$ defined in~\cref{{surrogate_bar_f}}, except its matrix inverse approximation. Indeed, the following~\Cref{le:bounded_matrix} shows that the bias between $\overline{\mathcal{H}}(x,y)$ and $\overline{\nabla}f$ decreases exponentially with respect to $N$.

\begin{proposition}\label{le:bounded_matrix}
Under Assumptions~\ref{ass:lipschitz} to~\ref{ass:bounded_var}, we have 
\begin{equation}
    \left\|\overline{\mathcal{H}}(x^k,y^{k+1})-\overline{\nabla}f(x^k,y^{k+1}))\right\|\leq b,\nonumber
\end{equation}
where {$b=\kappa_g \ell_{f,1}\left(\left(\kappa_g -1\right)/\kappa_g\right)^N$} and {$N$ is the input parameter to~\Cref{alg:HDMB-Partial}}.
\end{proposition}

The following two propositions explore the bounded variances of $\mathcal{H}_i$. Particularly, the $\mathcal{O}(1/n)$ factor of the bounded variance of the average of $\{\mathcal{H}_i\}_{i\in[n]}$ is presented in~\Cref{le:bounded_hyper_gradient_over_n}. Such a property highly relies on the independence among all the hypergradient estimators and plays an essential role in establishing the linear speedup. This is a key property that can be achieved by our proposed minibatch SGD and PHE algorithms and is missing in FedNest \citep{tarzanagh2022fednest} in the non-i.i.d setting.
\begin{proposition}\label{le:bounded_hyper_gradient}
Suppose Assumptions~\ref{ass:lipschitz} to~\ref{ass:bounded_var} hold for all $i\in [n]$. Then, we have 
\begin{align}
    \mathbb{E}\left[||\mathcal{H}_i(x^k,y^{k+1})-\overline{\mathcal{H}}(x^k,y^{k+1})||^2\right]&\leq \tilde{\sigma}_f^2,\nonumber\\
\mathbb{E}\left[||\mathcal{H}_i(x^k,y^{k+1})||^2 |\mathcal{F}^k\right]\leq \tilde{D}_f^2,\nonumber
\end{align}
where the constants $ \tilde{\sigma}_f^2$ and $\tilde{D}_f^2$ are given by 
\begin{align}
    \tilde{\sigma}_f^2:=& \sigma_f^2+\frac{3}{\mu_g^2}\left[(\sigma_f^2+\ell_{f,0}^2)(\sigma_{g,2}^2+2\ell_{g,1}^2)+\sigma_f^2\ell_{g,1}^2\right]
    =\mathcal{O}(\kappa_g^2),\nonumber\\
    \tilde{D}_f^2 :=& \big(\ell_{f,0}+\frac{\ell_{f,0}\ell_{g,1}}{\mu_g}+\frac{\ell_{f,1}\ell_{g,1}}{\mu_g}\big)^2+\tilde{\sigma}_f^2=\mathcal{O}(\kappa_g^2).\nonumber
\end{align}
\end{proposition}

\begin{proposition}
\label{le:bounded_hyper_gradient_over_n}
Under Assumptions~\ref{ass:lipschitz} to~\ref{ass:bounded_var}, we have 
\begin{align*}
    &\mathbb{E}\left[\left\|\frac{1}{n}\sum_{i=1}^n(\mathcal{H}_i(x^k,y^{k+1})-\overline{\mathcal{H}}(x^k,y^{k+1}))\right\|^2\right]\leq \frac{\tilde{\sigma}_f^2}{n},
\end{align*}
where {\small$\tilde{\sigma}_f:= \sigma_f^2+\frac{3}{\mu^2_g}\left((\sigma_f^2+\ell_{f,0}^2)(\sigma_{g,2}^2+2\ell_{g,1}^2)+\sigma_f^2\ell_{g,1}^2\right)$}. 
\end{proposition}

We next characterize the convergence and complexity performance of the proposed algorithm. 
\vspace{0.2cm}
\begin{theorem}\label{thm:full_worker}
Suppose Assumptions~\ref{ass:lipschitz} to~\ref{ass:bounded_global_variance} hold and set 
\begin{align}
    \alpha_k &= \min \left\{\hat{\alpha}_1,\hat{\alpha}_2,\sqrt{\frac{n}{K}} \hat{\alpha}_3\right\},
    \nonumber
    \\
    \beta_{k,t} &= \Big(\frac{5M_f L_y}{\mu_g}+\frac{\eta L_{yx}\tilde{D}^2_f\hat{\alpha}_1}{2n\mu_g}\Big)\frac{\alpha_k}{T},
    \nonumber
\end{align}
for some positive constants $\hat{\alpha}_i$, $i=1,2,3$ independent of $K$, where the definition of the constant parameters $M_f, L_y, \eta, L_{yx}, \tilde{D}^2_f $ can be found in the appendix. Then, for any $K \geq 1$, the iterates $\left\{(x^k,y^k)\right\}_{k\geq 0}$ generated by~\Cref{alg:HDMB-Partial} satisfy
\begin{align*}
     \frac{1}{K} \sum_{k=0}^{K-1}& \mathbb{E} \left[\left\|\nabla f(x^k)\right\|^2\right] =
    \mathcal{O}\Big(\frac{\hat{\alpha}_3+\hat{\alpha}_3^{-1}}{\sqrt{nK}}+\frac{1}{\min(\hat{\alpha}_1,\hat{\alpha}_2)K}+{b^2}\Big).\nonumber
\end{align*}
where $b=\kappa_g l_{f,1}\left(\left(\kappa_g -1\right)/\kappa_g\right)^N$ and $N$ is the controlling input parameter to~\Cref{alg:PHE_n}, 
\end{theorem}
\Cref{thm:full_worker} shows that for any given inner loop $T$, with a proper choice of the step sizes $\alpha_k, \beta_{k,t}$ and hyperparameters, the proposed FedMBO algorithm converges with a sub-linear rate. Moreover, the major term in the error bound $ \mathcal{O}\Big(\frac{1}{\sqrt{nK}})$ has a linear speedup w.r.t.~the number $n$ of the participating clients.
\begin{remark}
    Our theoretical analysis is mainly conducted on the case of partial client participation, i.e. $n<m$. For the full clients participation scenario, the analysis is easier and similar results (constants slightly different) can be obtained by following the proof steps in \Cref{appen:main}. 
\end{remark}
\begin{corollary}\label{corollary}
Under the same conditions as in~\Cref{thm:full_worker}, if we set $N=\Omega(\kappa_g \log K)$ and $ST=\Omega(\kappa^4_g)$, then 
\begin{align*}
    \frac{1}{K} \sum_{k=0}^{K-1} \mathbb{E} \left[\left\|\nabla f(x^k)\right\|^2\right] =\mathcal{O}\left(\frac{\kappa_g^{5/2}}{\sqrt{nK}}+\frac{\kappa_g^3}{K}+\frac{\kappa_g^{7/2}\sqrt{n}}{K^{3/2}}\right).
\end{align*}
In addition, we need $K=\mathcal{O}(\kappa^5_g\epsilon^{-2}/n)$ to achieve an $\epsilon$-accurate stationary point.
\end{corollary}

To achieve $\epsilon$-optimal solution, the samples we require in $\xi$ and $\theta$ are $\mathcal{O}(\kappa_g^9\epsilon^{-2})$ and $\mathcal{O}(\kappa_g^5\epsilon^{-2})$ respectively. Compared with FedNest~\citep{tarzanagh2022fednest} in the non-i.i.d.~setting, our complexity has the same dependence on $\kappa$ and $\epsilon$, but a better dependence on $n$ due to the linear speedup. As far as we know, this is the first linear speedup result for non-i.i.d.~federated bilevel optimization.



\begin{figure*}[t]
\vspace{-7pt}
    \centering
     {\includegraphics[scale=0.33]{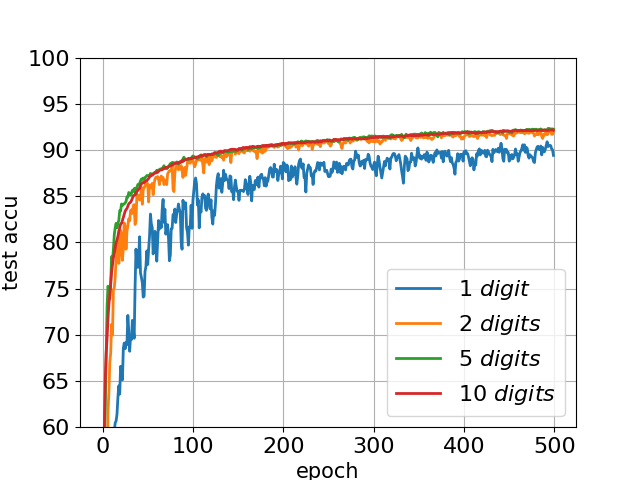}}
     {\includegraphics[scale=0.33]{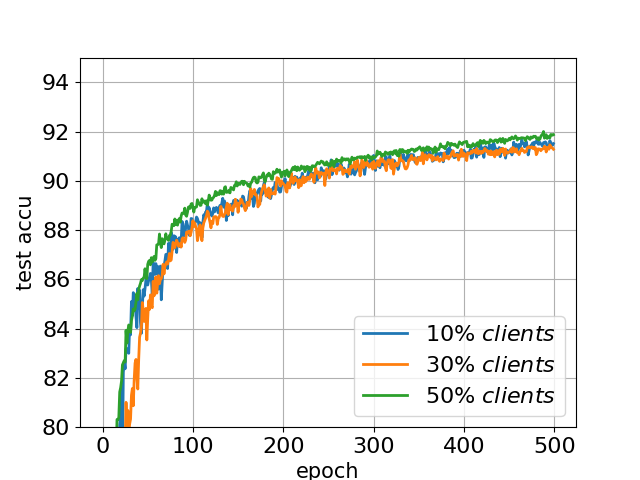}}
     {\includegraphics[scale=0.33]{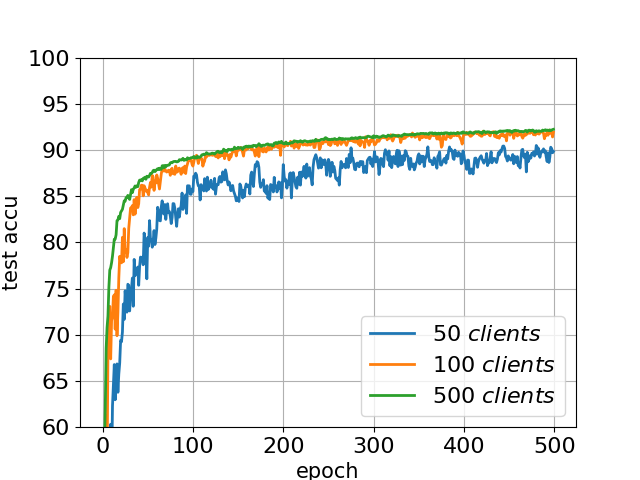}}
     {\includegraphics[scale=0.33]{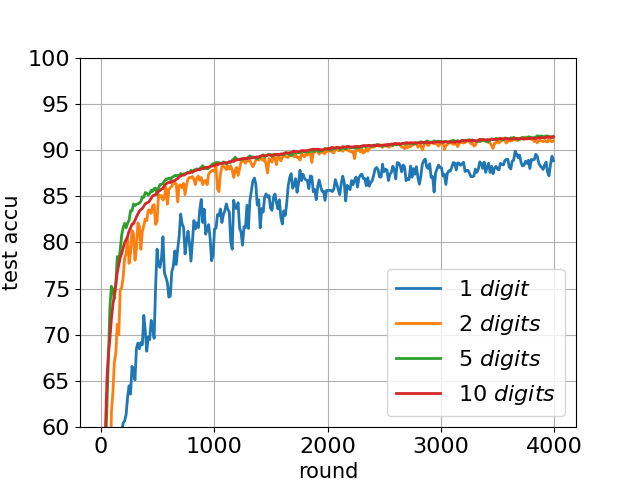}}
     {\includegraphics[scale=0.33]{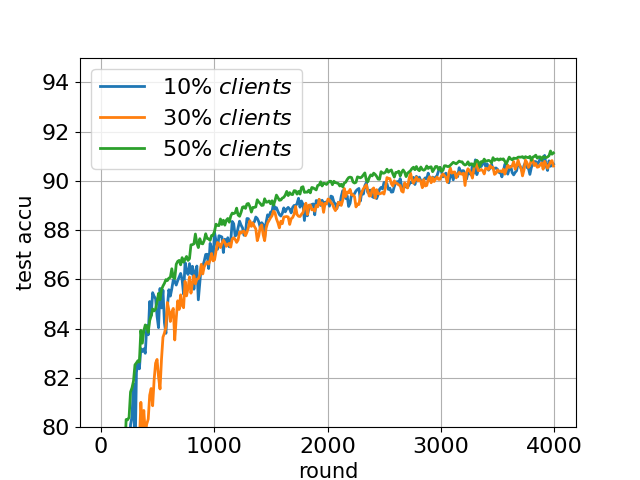}}
     {\includegraphics[scale=0.33]{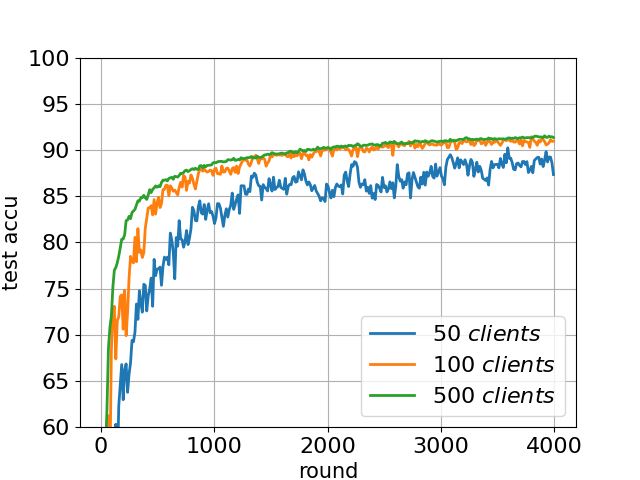}}
    \caption{\textbf{Left column:} Comparison among different levels of heterogeneity. \textbf{Middle column:} Comparison between different numbers of total clients. \textbf{Right column:} Comparison among different sampling rates. The first row plots the test accuracy against the epoch. The second plots the test accuracy against the number of communication rounds.}\label{fig:casestudy}
\end{figure*}

\begin{figure*}[t]
\vspace{-7pt}
    \centering
     {\includegraphics[scale=0.33]{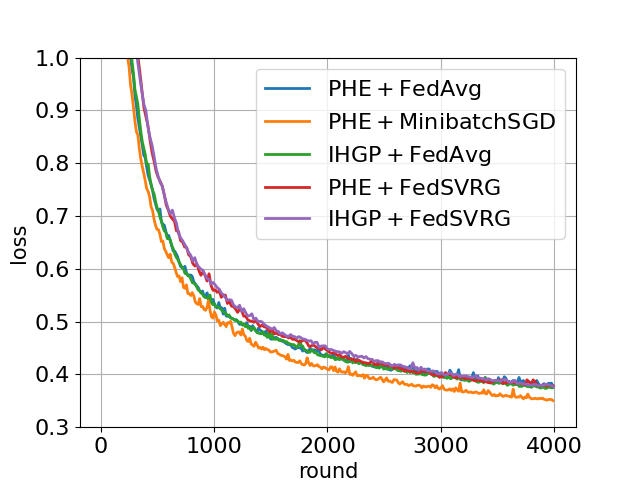}}
     {\includegraphics[scale=0.33]{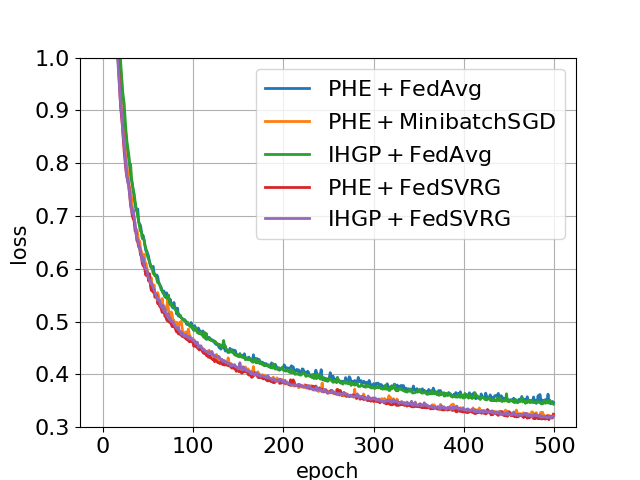}}
     {\includegraphics[scale=0.33]{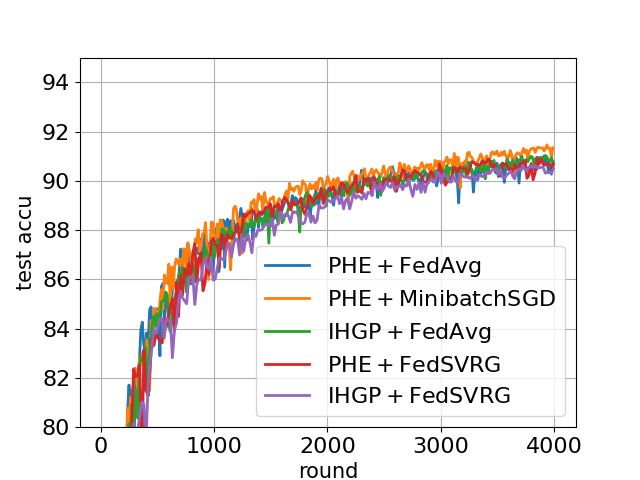}}
    \caption{Comparison between our PHE 
    with IHGP~\citep{tarzanagh2022fednest}
    under different lower-level optimizers.}.\label{fig:fncompare}\vspace{-15pt}
\end{figure*}

\section{Experiments}\label{exp}

In this section, we conduct experiments on hyper-representation, which is an important problem in multi-task machine learning, to validate our theoretical results.  
We focus on the hyper-representation problem in the federated setting, which can be formulated as
\begin{align*}
    &\min_{\phi} \mathcal{L}_{\mathcal{D}_v} (\phi, \omega^*) = \frac{1}{m} \sum_{i=1}^m \frac{1}{|\mathcal{D}_v^i|} \sum_{\xi \in \mathcal{D}_v^i} \mathcal{L} (\phi, \omega ^*; \xi)  \nonumber\\
    &\text{s.t. } \omega ^* = \argmin_{\omega} \frac{1}{m} \sum_{i=1}^m \frac{1}{|\mathcal{D}_t^i|} \sum_{\zeta \in \mathcal{D}_t^i} \mathcal{L} (\phi, \omega ; \zeta),
\end{align*}
where $\mathcal{D}_t^i$ and $\mathcal{D}_v^i$ are 
the training and the validation datasets respectively. 
Specifically, the upper-level problem learns the shared hyper feature representations using the validation data, and  the lower-level objective learns the prediction head for each client on the training data. In all experiments, we use a multi-layer perceptron (MLP) with 2 linear layers and 1 ReLU activation layer as our model architecture and focus on the heterogeneous case with non-i.i.d.~datasets. All experiments  are implemented in Python 3.7 on a Linux server with an Nvidia GeForce RTX 2080ti GPU.

\subsection{Case Studies}
In this section, we conduct experiments on several case studies to demonstrate the efficiency of our proposed algorithm. We first study the impact of heterogeneity in each client's dataset. We fix the client sampling ratio to $10\%$, and the number of clients to be $100$ and sample the dataset in a digit-based manner. In particular, the whole MNIST dataset is split into 10 subsets, where each subset contains all images with the same digit. The data in each client is sampled from a certain number of subsets. In a 2-digit case, for each client, we first randomly pick 2 digits, and then sample data from the images with these two digits. Note that the 10-digit case is equivalent to the homogeneous case. In this way, the number of digits measures the degree of heterogeneity. The result is summarized in the left column of~\Cref{fig:casestudy}. The proposed algorithm performs the worst in the 1-digit case with the highest data heterogeneity, and the performance is improved as we increase the number of digits due to the reduced data heterogeneity. This demonstrates the negative impact of data heterogeneity on the convergence performance. 


Second, we study the impact of different client sampling ratios. We fix the 2-digit sampling strategy for each client and the total number of clients to be $100$. From the middle column of~\Cref{fig:casestudy},it is seen that the case of $50\%$ client sampling ratio performs the best. Therefore, increasing the sampling ratio helps the performance of our algorithm.

Finally, we test the impact of different numbers of total clients. We fix a 2-digit sampling strategy for each client and the client sampling ratio to be $10\%$. We select $n \in \{50, 100, 500\}$ for the test. As shown in the right column of~\Cref{fig:casestudy}, the performance of our proposed algorithm becomes better as we increase the number of clients. 

\subsection{Comparison with FedNest}
We compare our approaches with FedNest \citep{tarzanagh2022fednest} in the non-i.i.d.~setting. We notice that \citealt{tarzanagh2022fednest} also proposed a Light FedNest (LFedNest) to reduce the communication rounds. However, LFedNest diverges in some of our non-i.i.d experiments and performs worse than the FedNest. So we focus on the comparison of FedNest and our proposed algorithm only. Two major components of the FedNest algorithm are IHGP for estimating the hypergradient and FedSVRG (or FedLin) for solving the lower-level problem. We compare the performance among different pairs of PHE, IHGP, and MinibatchSGD, FedSVRG, FedAvg. In this case, we set the number of total clients to 100 and the sampling ratio to be $10\%$. For the dataset of each client, we first sort the MNIST dataset according to their labels and then equally split it into 100 subsets and assign one subset to each client. In this way, we guarantee a high-level heterogeneity among all the clients. We set $T = 5$ for all cases and fine-tune the step sizes so that each setting achieves its best performance.

In Figure \ref{fig:fncompare}, we plot the loss and test accuracy against epoch and communication round respectively. The left figure plots the loss against the communication round. From the left figure, we conclude that among all the settings, the proposed PHE + MinibatchSGD converges the fastest. The middle figure plots the loss against data epochs and shows that the MinibatchSGD for the lower-level problem achieves similar performance to FedSVRG and both are better than the FedAvg Algorithm. The right figure shows that PHE + MinibatchSGD achieves the best test accuracy among all algorithms.

\section{Conclusion} 
This paper studies the federated bilevel optimization problem in the presence of data heterogeneity, and proposes a novel federated bilevel algorithm named FedMBO. We show that FedMBO is flexible with partial client participation and achieves a linear speedup for convergence. Numerical experiments are conducted to demonstrate the advantages of our proposed algorithms. We anticipate that our theoretical results and the proposed hypergradient estimator can be applied to other distributed scenarios such as decentralized bilevel optimization.






\bibliographystyle{ref_style.bst}
\bibliography{ref.bib}

\newpage
\appendix
\onecolumn

\noindent{\Large\bf Supplementary Materials}

\section{Supporting Lemmas} 
The following two lemmas are commonly used in the previous literature on (federated) bilevel optimization. We refer to the corresponding works for detailed proofs. 
\begin{lemma}\label{le:lipshitz}([\citealt{ghadimi2018approximation}, Lemma 2.2]) Under Assumptions~\ref{ass:lipschitz} and~\ref{ass:strong_convex}, we have 
\begin{align}
    ||\nabla \Phi(x_1) -\nabla \Phi(x_2)||&\leq L_f ||x_1 - x_2||,\label{eq:f_lip}\nonumber\\
    ||y^*(x_1)-y^*(x_2)||&\leq L_y ||x_1-x_2||, \nonumber
\end{align}
where 
\begin{align}
    L_f :=& l_{f,1}+\frac{l_{g,1}(l_{f,1}+M_f)}{\mu_g}+\frac{l_{f,0}}{\mu_g}(l_{g,2}+\frac{l_{g,1}l_{g,2}}{\mu_g})=\mathcal{O}(\kappa_g^3), 
    \nonumber\\
    L_y :=& \frac{l_{g,1}}{\mu_g}=\mathcal{O}(\kappa_g).
    \nonumber
\end{align}
For all $i\in [m]$, we have 
\begin{align*}
    ||\overline{\nabla} f_i(x_1, y)-\overline{\nabla} f_i(x_1,y^*(x_1))|| &\leq M_f ||y-y^*(x_1)||,\\
    ||\overline{\nabla}f_i(x_1,y)-\overline{\nabla} f_i(x_2,y)|| &\leq M_f||x_1-x_2||,
\end{align*}
where the constant $M_f$ is given by 
\begin{align}
    M_f:=l_{f,1}+\frac{l_{g,1}l_{f,1}}{\mu_g}+\frac{l_{f,0}}{\mu_g}(l_{g,2}+\frac{l_{g,1}l_{g,2}}{\mu_g})=\mathcal{O}(\kappa_g^2)
    \nonumber
\end{align}
and $\overline{\nabla} f_i$ is defined as 
\begin{equation}
    \overline{\nabla} f_i(x,y) := \nabla_x f_i(x,y)- \nabla^2_{xy} g(x,y)[\nabla^2_{yy} g(x,y)]^{-1}\nabla_y f_i(x,y).
    \nonumber
\end{equation}
\end{lemma}

\begin{proof}[\bf Proof of \Cref{le:lipshitz}]
The proof is similar to Lemma 2.2 in \citealt{ghadimi2018approximation}.
\end{proof}

\begin{lemma}\label{le:lipschitz_y}([\citealt{chen2021closing}, Lemma 2])
Under Assumptions~\ref{ass:lipschitz} to~\ref{ass:sto_sample}, we have 
\begin{align*}
    ||\nabla y^*(x_1)-\nabla y^*(x_2)||&\leq L_{yx} ||x_1-x_2||,
\end{align*}
where the constant $L_{yx}$ is given by 
\begin{align}
    L_{yx}:=\frac{l_{g,2}+l_{g,2}L_y}{\mu_g}+\frac{l_{g,1}}{\mu_g^2}(l_{g,2}+l_{g,2}L_y)=\mathcal{O}(\kappa_g^3).\nonumber
\end{align}
\end{lemma}
\begin{proof}[\bf Proof of \Cref{le:lipschitz_y}]
The proof is similar to Lemma 2 in \citealt{chen2021closing}.
\end{proof}

\section{Proof of Proposition in~\Cref{mainresults}}\label{prop:appendx}

\begin{proof}[\bf Proof of \Cref{le:bounded_matrix}]
    The independency of $c_i^0$, $c_i^{\max\{N_i\}+1}$, $c_i^{\ell}$, $\ell=1,...,n$, $N_i$ and their data points sample is guaranteed based on the algorithmic design in~\Cref{alg:PHE_n}. Therefore, we have 
    \begin{align}
    &\overline{\mathcal{H}}(x^k,y^{k+1}):=\mathbb{E}\left[\mathcal{H}_i(x^k,y^{k+1})|\mathcal{F}^k\right]
    \nonumber 
    \\
     =& \nabla_x f(x^k, y^{k+1})- \nabla^2_{xy} g(x^k, y^{k+1})\mathbb{E}\bigg[\frac{N}{l_{g,1}}\prod_{l=1}^{N_i}\Big(I - \frac{1}{\ell_{g, 1}}\nabla_{yy}^2 g_{C_i^l}(x^k, y^{k+1};\zeta^{i, l})\Big)\Big|\mathcal{F}^k\bigg]\nabla_y f(x^k, y^{k+1}). \nonumber 
\end{align}
From the definition of $\overline{\nabla}f$, we have 
\begin{align}
    &\left\|\overline{\mathcal{H}}(x^k,y^{k+1})-\overline{\nabla}f(x^k,y^{k+1}))\right\|\nonumber
    \\
    &\leq \left\|\nabla^2_{xy} g(x^k, y^{k+1})\right\|\cdot\left\|\mathbb{E}\bigg[\frac{N}{l_{g,1}}\prod_{l=1}^{N_i}\Big(I - \frac{1}{\ell_{g, 1}}\nabla_{yy}^2 g_{C_i^l}(x^k, y^{k+1};\zeta^{i, l})\Big)\Big|\mathcal{F}^k\bigg]-\left[\nabla^2_{yy}g(x,y)\right]^{-1}\right\| \nonumber
    \\&\quad\times\left\|\nabla_y f(x^k, y^{k+1})\right\|\nonumber
    \\
    &\leq \frac{\ell_{f,0}\ell_{g,1}}{\mu_g}\big(1-\frac{\mu_g}{\ell_g}\big)^N,\nonumber
\end{align}
where we have applied the Assumption~\ref{ass:lipschitz} and~\cref{Neumann_series} to the last inequality. 
Then, the proof is complete. 
\end{proof}

\begin{proof}[\bf Proof of \Cref{le:bounded_hyper_gradient}]
Following Lemma 1 in~\citealt{hong2020two}, we can derive 
\begin{align}
    \mathbb{E}\left[||\mathcal{H}_i(x^k,y^{k+1})-\overline{\mathcal{H}}(x^k,y^{k+1})||^2\right]&\leq \tilde{\sigma}_f,\label{bound_variance_hypergrad}
\end{align}
where $\tilde{\sigma}_f^2:= \sigma_f^2+\frac{3}{\mu_g^2}\left[(\sigma_f^2+\ell_{f,0}^2)(\sigma_{g,2}^2+2\ell_{g,1}^2)+\sigma_f^2\ell_{g,1}^2\right]$. 
From the definition in~\Cref{surrogate_bar_f}, we have 
\begin{align}
    \left\|\overline{\nabla} f(x,y)\right\| &= \left\|\nabla_x f(x,y)-\nabla^2_{xy} g(x,y)[\nabla^2_{yy} g(x,y)]^{-1}\nabla_y f(x,y)\right\|\nonumber
    \\
    & \leq \left\|\nabla_x f(x,y)\right\|+\left\|\nabla^2_{xy} g(x,y)\right\|\left\|[\nabla^2_{yy} g(x,y)]^{-1}\right\|\left\|\nabla_y f(x,y)\right\|\nonumber 
    \\
    &\leq \ell_{f,0}+\frac{\ell_{f,0}\ell_{g,1}}{\mu_g},\label{bar_surro}
\end{align}
where the last inequality comes from Assumptions~\ref{ass:lipschitz} and~\ref{ass:strong_convex}. 

Now we derive the bound of $\mathcal{H}_i(x^k,y^{k+1})$ as follows,
\begin{align}
    \mathbb{E}\left[||\mathcal{H}_i(x^k,y^{k+1})||^2 |\mathcal{F}^k\right]&=\left\|\overline{\mathcal{H}}(x^k,y^{k+1})\right\|^2+\mathbb{E}\left[||\mathcal{H}_i(x^k,y^{k+1})-\overline{\mathcal{H}}(x^k,y^{k+1})||^2|\mathcal{F}^k\right]\nonumber
    \\
    &\leq \left(\left\|\overline{\mathcal{H}}(x^k,y^{k+1})-\overline{\nabla} f(x,y)\right\|+ \left\|\overline{\nabla} f(x,y)\right\|\right)^2+\tilde{\sigma}^2_f\nonumber
    \\
    &\leq \big(\kappa_g \ell_{f,1}\left(\left(\kappa_g -1\right)/\kappa_g\right)^N+\ell_{f,0}+\ell_{g,1}\frac{1}{\mu_g}\ell_{f,0}\big)^2+\tilde{\sigma}^2_f\nonumber
    \\
    &\leq \big(\ell_{f,0}+\frac{\ell_{f,0}\ell_{g,1}}{\mu_g}+\frac{\ell_{f,1}\ell_{g,1}}{\mu_g}\big)^2+\tilde{\sigma}_f^2,\nonumber
\end{align}
where the first inequality is based on the result of~\cref{bound_variance_hypergrad} and the second inequality is based on~\Cref{{le:bounded_matrix}} and~\cref{bar_surro}. 
Then, the proof is complete.
\end{proof}

\begin{proof}[\bf Proof of \Cref{le:bounded_hyper_gradient_over_n}]
{Note that if we choose to sample the clients \textit{with replacement} in~\Cref{alg:PHE_n}, then apparently $\left\{\mathcal{H}_i\right\}$ are pairwise independent random variables (refer to~\Cref{fig:samplingdiagram}). From~\Cref{le:bounded_hyper_gradient}, we have the variances of $\left\{\mathcal{H}_i\right\}_{i=1,...,n}$ are  bounded by a constant $\tilde{\sigma}^2_f$. Therefore, we have}
\begin{subequations}
\begin{align}
    &\mathbb{E}\Big\|\frac{1}{n}\sum_{i=1}^n(\mathcal{H}_i(x^k,y^{k+1})-\overline{\mathcal{H}}(x^k,y^{k+1}))\Big\|^2
    \nonumber 
    \\
    &
    = \frac{1}{n^2}\sum_{i=1}^n \mathbb{E}\left\|\mathcal{H}_i(x^k,y^{k+1})-\overline{\mathcal{H}}(x^k,y^{k+1})\right\|^2\nonumber
    \\
    &+\frac{1}{n^2}\sum_{1\leq i\neq j\leq n} \mathbb{E}\Big[\mathbb{E}\Big\langle \mathcal{H}_i(x^k,y^{k+1})-\overline{\mathcal{H}}(x^k,y^{k+1}),\mathcal{H}_j(x^k,y^{k+1})-\overline{\mathcal{H}}(x^k,y^{k+1})\Big\rangle |\mathcal{F}^k \Big] \nonumber
    \\
    &\leq \frac{\tilde{\sigma}^2_f}{n},
    \nonumber
\end{align}
\end{subequations}
where the last inequality follows because $\mathcal{H}_i-\overline{\mathcal{H}}$ and $\mathcal{H}_j-\overline{\mathcal{H}}$ are independent  conditioning on $\mathcal{F}^k$.  
Then, the proof is complete.
\end{proof}

\section{Convergence Proofs}\label{appen:main}
\begin{proof}[\bf Proof of \Cref{le:implicitfunction}]
This result has been well-known in the literature on bilevel optimization. See, e.g.,~\citealt{ghadimi2018approximation} for its proof. 
\end{proof}

\begin{lemma}\label{le:descent}
Suppose Assumptions~\ref{ass:lipschitz} to~\ref{ass:bounded_var} hold, \Cref{alg:HDMB-Partial} guarantees:
\begin{align}
\mathbb{E}\left[f(x^{k+1})\right]-\mathbb{E}[f(x^k)]
&\leq \alpha_k M_f^2 \mathbb{E}\left[\left\|y^{k+1}-y^{*}(x^k)\right\|^2\right]+(\alpha_k^2 L_f-\frac{\alpha_k}{2})\mathbb{E} \left[\left\| \overline{\mathcal{H}}(x^k,y^{k+1})\right\|^2\right]
\nonumber
\\ &-\frac{\alpha_k}{2} \mathbb{E}\left[\left\|\nabla f(x^k)\right\|^2\right]+\alpha_k b^2 +\frac{\alpha_k^2 L_f \tilde{\sigma}_f^2}{n}.
\nonumber
\end{align}
\end{lemma}
\begin{proof}[\bf Proof of \Cref{le:descent}]
Based on the smoothness property of the objective function $\Phi$ established in~\Cref{le:lipshitz}, we have 
\begin{align}\label{eq:descent_main}
& \mathbb{E}\left[f(x^{k+1})\right]-\mathbb{E}\left[f(x^k)\right] \nonumber
 \\\leq & \mathbb{E}\left[\left\langle x^{k+1} - x^k, \nabla f(x^k) \right\rangle\right] +\frac{L_f}{2} \mathbb{E}\left[\left\|x^{k+1} -x^k\right\|^2\right] \nonumber
    \\=& -\mathbb{E} \left[\left\langle \frac{1}{n}\sum_{i=1}^n \alpha_k \mathcal{H}_i(x^k,y^{k+1}),\nabla f(x^k)\right\rangle\right] +\frac{L_f}{2}\mathbb{E}\left[\left\|\frac{1}{n}\sum_{i=1}^n \alpha_k \mathcal{H}_i(x^k,y^{k+1})\right\|^2\right]. 
\end{align}
To bound the first term of~\cref{eq:descent_main}, we have 
\begin{align}
    & -\mathbb{E} \bigg[\Big\langle \frac{1}{n}\sum_{i=1}^n \alpha_k \mathcal{H}_i(x^k,y^{k+1}),\nabla f(x^k)\Big\rangle\bigg]
    \nonumber\\
    = & -\mathbb{E}\left[\frac{1}{n}\sum_{i=1}^n \alpha_k \mathbb{E}\left[\langle \mathcal{H}_i(x^k,y^{k+1}),\nabla f(x^k)\rangle | \mathcal{F}^k\right] \right]
    \nonumber
    \\ =& -\mathbb{E}\left[\left\langle \alpha_k \overline{\mathcal{H}}(x^k,y^{k+1}),\nabla f(x^k)\right\rangle\right] 
    \nonumber
    \\ =& -\frac{\alpha_k}{2} \mathbb{E} \left[\left\| \overline{\mathcal{H}}(x^k,y^{k+1})\right\|^2\right] -\frac{\alpha_k}{2} \mathbb{E}\left[\left\| \nabla f(x^k)\right\|^2\right]+\frac{\alpha_k}{2} \mathbb{E} \left[\left\|  \overline{\mathcal{H}}(x^k,y^{k+1})-\nabla f(x^k)\right\|^2\right] 
    \nonumber 
    \\ =& -\frac{\alpha_k}{2} \mathbb{E} \left[\left\| \overline{\mathcal{H}}(x^k,y^{k+1})\right\|^2\right] -\frac{\alpha_k}{2} \mathbb{E}\left[\left\| \nabla f(x^k)\right\|^2\right]
    \nonumber 
    \\ & + \frac{\alpha_k}{2} \mathbb{E}\left[\left\|  \overline{\mathcal{H}}(x^k,y^{k+1})-\overline{\nabla} f(x^k,y^{k+1})+\overline{\nabla} f(x,y^{k+1})-\nabla f(x^k)\right\|^2\right]
    \nonumber 
    \\ \leq & -\frac{\alpha_k}{2} \mathbb{E} \left[\left\| \overline{\mathcal{H}}(x^k,y^{k+1})\right\|^2\right] -\frac{\alpha_k}{2} \mathbb{E}\left[\left\| \nabla f(x^k)\right\|^2\right]
    \nonumber 
    \\ & + \alpha_k\mathbb{E}\left[\left\|\overline{\mathcal{H}}(x^k,y^{k+1})-\overline{\nabla}f(x^k,y^{k+1}))\right\|^2\right] +\alpha_k\mathbb{E}\left[\left\|\overline{\nabla}f(x^k,y^{k+1})-\nabla f(x^k)\right\|^2\right]\nonumber
    \\ \leq & -\frac{\alpha_k}{2} \mathbb{E} \left[\left\| \overline{\mathcal{H}}(x^k,y^{k+1})\right\|^2\right] -\frac{\alpha_k}{2} \mathbb{E}\left[\left\| \nabla f(x^k)\right\|^2\right]+ {\alpha_k b^2} +\alpha_k M_f^2\mathbb{E}\left[\left\|y^{k+1}-y^*(x^k)\right\|^2\right],\nonumber 
\end{align}
where the last inequality is due to ~\Cref{le:lipshitz} and~\Cref{le:bounded_matrix}. The second term of~\cref{eq:descent_main} can be bounded as
\begin{align}
&\frac{L_f}{2} \mathbb{E}\left[\left\| \frac{1}{n}\sum_{i=1}^n \alpha_k \mathcal{H}_i(x^k,y^{k+1})\right\|^2\right]
\nonumber\\
=& \frac{\alpha_k^2 L_f}{2} \mathbb{E}\left[\left\|\frac{1}{n}\sum_{i=1}^n(\mathcal{H}_i(x^k,y^{k+1})-\overline{\mathcal{H}}(x^k,y^{k+1})+\overline{\mathcal{H}}(x^k,y^{k+1}))\right\|^2\right]
\nonumber
\\\leq & \alpha_k^2 L_f \mathbb{E}\left[\left\|\frac{1}{n}\sum_{i=1}^n(\mathcal{H}_i(x^k,y^{k+1})-\overline{\mathcal{H}}(x^k,y^{k+1}))\right\|^2\right]+\alpha_k^2 L_f \mathbb{E} \left[\left\| \overline{\mathcal{H}}(x^k,y^{k+1})\right\|^2\right]
\nonumber
\\ =& {\frac{\alpha_k^2 L_f \tilde{\sigma}_f^2}{n}}+\alpha_k^2 L_f \mathbb{E} \left[\left\| \overline{\mathcal{H}}(x^k,y^{k+1})\right\|^2\right],\nonumber
\end{align}
where we use ~\Cref{le:bounded_hyper_gradient_over_n} in the last equality. Combining the above inequalities yields 
\begin{align}
\mathbb{E}[f(x^{k+1})]-\mathbb{E}[f(x^k)]
&\leq \alpha_k M_f^2 \mathbb{E}\left[\left\|y^{k+1}-y^{*}(x^k)\right\|^2\right]+(\alpha_k^2 L_f-\frac{\alpha_k}{2})\mathbb{E} \left[\left\|  \overline{\mathcal{H}}(x^k,y^{k+1})\right\|^2\right]
\nonumber
\\ &-\frac{\alpha_k}{2} \mathbb{E}\left[\left\|\nabla f(x^k)\right\|^2\right]+\alpha_k b^2 +\frac{\alpha_k^2 L_f \tilde{\sigma}_f^2}{n}.
\nonumber
\end{align}
Then, the proof is complete.
\end{proof}

\begin{lemma}\label{le:error_FEDINN_pre}
Suppose Assumptions~\ref{ass:lipschitz} to~\ref{ass:bounded_global_variance} hold and $0<\beta_{k,t}\leq\frac{1}{2 l_{g,1}}$, the iterates of \Cref{alg:HDMB-Partial} guarantees:
\begin{equation}\label{eq:error_FEDINN_pre}
    \mathbb{E} \left[\left\| y^{k+1}-y^{*}(x^k)\right\|^2\right] \leq \left(\prod_{t=0}^{T-1}(1-\beta_{k,t} \mu_g)\right) \mathbb{E}\left[\left\|y^k-y^*(x^k)\right\|^2\right] +\frac{4(\sigma_{g,1}^2+\sigma_g^2)}{nS}\sum_{t=0}^{T-1}\beta_{k,t}^2.\nonumber
\end{equation}
\end{lemma}
\begin{proof}[\bf Proof of \Cref{le:error_FEDINN_pre}]
We first show
\begin{align}
    &\mathbb{E}\left\|\frac{1}{nS}\sum_{i=1}^n \sum_{s=0}^{S-1}\nabla_y g_{c^{k,t}_i}\left(x^k,y^*(x^k); \xi_{i, s}^{k, t}\right)\right\|^2\leq \frac{2(\sigma_{g,1}^2+\sigma_g^2)}{nS}.\label{lemma:result_1}
\end{align}
By the algorithm update, we have 
\begin{align}
    &\mathbb{E}\left\|\frac{1}{nS}\sum_{i=1}^n \sum_{s=0}^{S-1}\nabla_y g_{c^{k,t}_i}\left(x^k,y^*(x^k); \xi_{i, s}^{k, t}\right)\right\|^2
    \nonumber
    \\
    =&\frac{1}{n^2}\sum_{i=1}^n \mathbb{E}\left\|\frac{1}{S}\sum_{s=0}^{S-1}\nabla_y g_{c^{k,t}_i}\left(x^k,y^*(x^k); \xi_{i, s}^{k, t}\right)\right\|^2
    \nonumber
    \\
    &+\frac{1}{n^2}\sum_{1\leq i\neq j\leq n}\mathbb{E}\left\langle \frac{1}{S}\sum_{s=0}^{S-1}\nabla_y g_{c^{k,t}_i}\left(x^k,y^*(x^k); \xi_{i, s}^{k, t}\right),\frac{1}{S}\sum_{s=0}^{S-1}\nabla_y g_{c^{k,t}_j}\left(x^k,y^*(x^k); \xi_{j, s}^{k, t}\right) \right\rangle
    \nonumber
    \\
    =&\frac{1}{n^2S^2}\sum_{i=1}^n\sum_{s=0}^{S-1}\mathbb{E}\left\|\nabla_y g_{c^{k,t}_i}\left(x^k,y^*(x^k); \xi_{i, s}^{k, t}\right)\right\|^2 
    \nonumber
    \\
    &+\frac{1}{n^2}\sum_{1\leq i\neq j\leq n}\mathbb{E}\left\langle \frac{1}{S}\sum_{s=0}^{S-1}\nabla_y g_{c^{k,t}_i}\left(x^k,y^*(x^k); \xi_{i, s}^{k, t}\right),\frac{1}{S}\sum_{s=0}^{S-1}\nabla_y g_{c^{k,t}_j}\left(x^k,y^*(x^k); \xi_{j, s}^{k, t}\right) \right\rangle
    \nonumber
    \\
    \leq&\frac{2}{n^2S^2}\sum_{i=1}^n\sum_{s=0}^{S-1}\mathbb{E}\left\|\nabla_y g_{c^{k,t}_i}\left(x^k,y^*(x^k); \xi_{i, s}^{k, t}\right)-\nabla_y g_{c^{k,t}_i}\left(x^k,y^*(x^k)\right)\right\|^2
    \nonumber
    \\
    &+\frac{2}{n^2S^2}\sum_{i=1}^n\sum_{s=0}^{S-1}\mathbb{E}\left\|\nabla_y g_{c^{k,t}_i}\left(x^k,y^*(x^k)\right)-\nabla_y g\left(x^k,y^*(x^k)\right)\right\|^2
    \nonumber
    \\
    &+\frac{1}{n^2}\sum_{1\leq i\neq j\leq n}\mathbb{E}\left\langle \frac{1}{S}\sum_{s=0}^{S-1}\nabla_y g_{c^{k,t}_i}\left(x^k,y^*(x^k); \xi_{i, s}^{k, t}\right),\frac{1}{S}\sum_{s=0}^{S-1}\nabla_y g_{c^{k,t}_j}\left(x^k,y^*(x^k); \xi_{j, s}^{k, t}\right) \right\rangle
    \nonumber
    \\
    \leq& \frac{2(\sigma_{g,1}^2+\sigma_g^2)}{nS}
   +\frac{1}{n^2}\sum_{1\leq i\neq j\leq n}\mathbb{E}\left\langle \frac{1}{S}\sum_{s=0}^{S-1}\nabla_y g_{c^{k,t}_i}\left(x^k,y^*(x^k); \xi_{i, s}^{k, t}\right),\frac{1}{S}\sum_{s=0}^{S-1}\nabla_y g_{c^{k,t}_j}\left(x^k,y^*(x^k); \xi_{j, s}^{k, t}\right) \right\rangle,\label{lemma:inner_product}
\end{align}
where the second equality comes from the pairwise independence between $\xi$, and the last inequality is due to Assumptions~\ref{ass:bounded_var} and~\ref{ass:bounded_global_variance}. We next show the second term in~\cref{lemma:inner_product} equal to zero:
\begin{align}
    &\sum_{1\leq i\neq j\leq n}\mathbb{E}\left\langle \frac{1}{S}\sum_{s=0}^{S-1}\nabla_y g_{c^{k,t}_i}\left(x^k,y^*(x^k); \xi_{i, s}^{k, t}\right),\frac{1}{S}\sum_{s=0}^{S-1}\nabla_y g_{c^{k,t}_j}\left(x^k,y^*(x^k); \xi_{j,s}^{k, t}\right) \right\rangle
    \nonumber
    \\
    &=\sum_{1\leq i\neq j\leq n}\mathbb{E}\left\langle \nabla_y g_{c^{k,t}_i}\left(x^k,y^*(x^k)\right),\nabla_y g_{c^{k,t}_j}\left(x^k,y^*(x^k) \right) \right\rangle
    \nonumber
    \\
    &=\sum_{1\leq i\neq j\leq n}\sum_{1\leq p\neq q\leq m}\mathbb{E}\left[\left\langle \nabla_y g_{c^{k,t}_i}\left(x^k,y^*(x^k)\right),\nabla_y g_{c^{k,t}_j}\left(x^k,y^*(x^k) \right) \right\rangle| c_i^{k,t}=p,c_j^{k,t}=q\right]\cdot\mathbb{P}(c_i^{k,t}=p,c_j^{k,t}=q)
    \nonumber
    \\
    &=\sum_{1\leq p\neq q\leq m}\mathbb{E}\left[\left\langle \nabla_y g_{p}\left(x^k,y^*(x^k)\right),\nabla_y g_{q}\left(x^k,y^*(x^k) \right) \right\rangle\right]\sum_{1\leq i\neq j\leq n}\mathbb{P}(c_i^{k,t}=p,c_j^{k,t}=q)
    \nonumber
    \\
    &=\sum_{1\leq i\neq j\leq n}\mathbb{P}(c_i^{k,t}=1,c_j^{k,t}=2)\sum_{1\leq p\neq q\leq m}\mathbb{E}\left[\left\langle \nabla_y g_{p}\left(x^k,y^*(x^k)\right),\nabla_y g_{q}\left(x^k,y^*(x^k) \right) \right\rangle\right]
    \nonumber
    \\
    &\leq \sum_{1\leq i\neq j\leq n}\mathbb{P}(c_i^{k,t}=1,c_j^{k,t}=2) \mathbb{E}\left\|\sum_{p=1}^m\nabla g_p(x^k,y^*(x^k))\right\|^2 = 0,
    \nonumber
\end{align}
where the third equality is based on the fact that $\mathbb{P}(c_i^{k,t}=p,c_j^{k,t}=q)$ is constant cross different combination $(p,q)$ and the last equality is due the optimality condition of the lower level problem. Next we show that for any $t\in \{0,..., T-1\}$, 
\begin{equation}
    \mathbb{E} \left[\left\| y^{k,t+1}-y^{*}(x^k)\right\|^2\right] \leq (1-\beta_{k,t} \mu_g) \mathbb{E}\left[\left\|y^{k,t}-y^*(x^k)\right\|^2\right] +\frac{4\beta_{k,t}^2(\sigma_{g,1}^2+\sigma_g^2)}{nS}.\nonumber
\end{equation}
Note that 
\begin{align}
    &\mathbb{E} \left\|y^{k,t+1}-y^*(x^{k})\right\|^2
    \nonumber
    \\
    =&\mathbb{E}\left\|y^{k,t}-\frac{\beta_{k,t}}{n}\sum_{i=1}^n G_i^{k,t}-y^*(x^{k})\right\|^2\nonumber 
    \\
    =&\mathbb{E}\left\|y^{k,t}-y^*(x^k)\right\|^2-2\beta_{k,t}\mathbb{E}\left\langle\frac{1}{n}\sum_{i=1}^n G_i^{k,t},y^{k,t}-y^*(x^k)\right\rangle+\beta_{k,t}^2\mathbb{E}\left\|\frac{1}{n}\sum_{i=1}^n G_i^{k,t}\right\|^2
    \nonumber
    \\
    =&\mathbb{E}\left\|y^{k,t}-y^*(x^k)\right\|^2-2\beta_{k,t}\mathbb{E}\left\langle\nabla_y g(x^k, y^{k,t}),y^{k,t}-y^*(x^k)\right\rangle+\beta_{k,t}^2\mathbb{E}\left\|\frac{1}{n}\sum_{i=1}^n G_i^{k,t}\right\|^2
    \nonumber
    \\
    \leq& (1-\beta_{k,t}\mu_g)\left\|y^{k,t}-y^*(x^k)\right\|^2-2\beta_{k,t} \mathbb{E}\left[g(x^k,y^{k,t})-g(x^k,y^*(x^k))\right]+\beta_{k,t}^2\mathbb{E}\left\|\frac{1}{n}\sum_{i=1}^n G_i^{k,t}\right\|^2,\label{lemma:y_diff}
\end{align}
where $G_i^{k,t}=\frac{1}{S}\sum_{s=0}^{S-1}\nabla_y g_{c_i^{k,t}}(x^k,y^{k,t};\xi_{i,s}^{k,t})$ as defined in~\Cref{alg:HDMB-Partial}.
We use the fact that $G_i^{k,t}$ is an unbiased gradient estimator in the third equality and employ the $\mu_g$-strong convexity
of $g(x, y)$ with respect to $y$ in the last inequality. To bound the last term in~\cref{lemma:y_diff}, we have
\begin{align}
    &\mathbb{E}\left\|\frac{1}{n}\sum_{i=1}^n G_i^{k,t}\right\|^2
    \nonumber
    \\
    =&\mathbb{E}\left\|\frac{1}{nS}\sum_{i=1}^n\sum_{s=0}^{S-1}\left[\nabla_y g_{c_i^{k,t}}(x^k,y^{k,t};\xi_{i,s}^{k,t})-\nabla_y g_{c_i^{k,t}}(x^k,y^{*}(x^k);\xi_{i,s}^{k,t})+\nabla_y g_{c_i^{k,t}}(x^k,y^{*}(x^k);\xi_{i,s}^{k,t})\right]\right\|^2
    \nonumber
    \\
    \leq& 2\mathbb{E}\left\|\frac{1}{nS}\sum_{i=1}^n\sum_{s=0}^{S-1}\left[\nabla_y g_{c_i^{k,t}}(x^k,y^{k,t};\xi_{i,s}^{k,t})-\nabla_y g_{c_i^{k,t}}(x^k,y^{*}(x^k);\xi_{i,s}^{k,t})\right]\right\|^2
    \nonumber
    \\
    &+2\mathbb{E}\left\|\frac{1}{nS}\sum_{i=1}^n\sum_{s=0}^{S-1}\left[\nabla_y g_{c_i^{k,t}}(x^k,y^{*}(x^k);\xi_{i,s}^{k,t})\right]\right\|^2
    \nonumber
    \\
    \leq& \frac{2}{nS}\sum_{i=1}^n\sum_{s=0}^{S-1}\mathbb{E}\left\|\nabla_y g_{c_i^{k,t}}(x^k,y^{k,t};\xi_{i,s}^{k,t})-\nabla_y g_{c_i^{k,t}}(x^k,y^{*}(x^k);\xi_{i,s}^{k,t})\right\|^2+\frac{4(\sigma_{g,1}^2+\sigma_g^2)}{nS}
    \nonumber
    \\
    \leq& \frac{4l_{g,1}}{nS}\sum_{i=1}^n\sum_{s=0}^{S-1}\mathbb{E}\left[g_{c_i^{k,t}}(x^k,y^{k,t};\xi_{i,s}^{k,t})-g_{c_i^{k,t}}(x^k,y^{*}(x^k);\xi_{i,s}^{k,t})-\left\langle \nabla_y g_{c_i^{k,t}}(x^k,y^{*}(x^k);\xi_{i,s}^{k,t}),y^{k,t}-y^{*}(x^k)\right\rangle\right]
    \nonumber
    \\
    &+\frac{4(\sigma_{g,1}^2+\sigma_g^2)}{nS}
    \nonumber
    \\=&4l_{g,1}\mathbb{E}\left[g\left(x^k,y^{k,t}\right)-g\left(x^k,y^*(x^k)\right)\right]+\frac{4(\sigma_{g,1}^2+\sigma_g^2)}{nS},
    \nonumber
\end{align}
where the second inequality uses the previous result of~\cref{lemma:result_1} and the third inequality uses Lemma 1 in~\citep{woodworth2020minibatch}. Plugging this into~\cref{lemma:y_diff} and enforcing $\beta_{k,t}\leq\frac{1}{2l_{g,1}}$ yield
\begin{align}
    &\mathbb{E} \left\|y^{k,t+1}-y^*(x^{k})\right\|^2
    \nonumber
    \\
    \leq& (1-\beta_{k,t}\mu_g)\left\|y^{k,t}-y^*(x^k)\right\|^2+2\beta_{k,t}(2\beta_{k,t}l_{g,1}-1) \mathbb{E}\left[g(x^k,y^{k,t})-g(x^k,y^*(x^k))\right]+\frac{4\beta_{k,t}^2(\sigma_{g,1}^2+\sigma_g^2)}{nS}
    \nonumber
    \\
    \leq& (1-\beta_{k,t}\mu_g)\left\|y^{k,t}-y^*(x^k)\right\|^2+\frac{4\beta_{k,t}^2(\sigma_{g,1}^2+\sigma_g^2)}{nS}.\label{lemma:y_diff_3}
\end{align}
Applying recursion on~\cref{lemma:y_diff_3}, we obtain 
\begin{equation*}
    \mathbb{E} \left[\left\| y^{k,T}-y^{*}(x^k)\right\|^2\right] \leq \left(\prod_{t=0}^{T-1}(1-\beta_{k,t} \mu_g)\right) \mathbb{E}\left[\left\|y^{k,0}-y^*(x^k)\right\|^2\right] +\frac{4(\sigma_{g,1}^2+\sigma_g^2)}{nS}\sum_{t=0}^{T-1}\beta_{k,t}^2,
\end{equation*}
which completes the proof.
\end{proof}
\begin{remark}
    In the case of full client participation and the clients are sampled without replacement, from the similar analysis above, we have 
\begin{equation}
    \mathbb{E} \left[\left\| y^{k+1}-y^{*}(x^k)\right\|^2\right] \leq \left(\prod_{t=0}^{T-1}(1-\beta_{k,t} \mu_g)\right) \mathbb{E}\left[\left\|y^k-y^*(x^k)\right\|^2\right] +\frac{4\sigma_{g,1}^2}{nS}\sum_{t=0}^{T-1}\beta_{k,t}^2.
    \nonumber
\end{equation}
Especially Assumption~\ref{ass:bounded_global_variance} is released for this scenario. 
\end{remark}

\begin{lemma}\label{le:error_FEDINN_post} 
Suppose Assumptions~\ref{ass:lipschitz} to~\ref{ass:bounded_var} hold, \Cref{alg:HDMB-Partial} guarantees:
\begin{align}
     \mathbb{E}\left[\left\|y^{k+1} -y^*(x^{k+1})\right\|^2\right] \leq & a_1(\alpha_k) \mathbb{E}\left[\left\|\overline{\mathcal{H}}(x^k,y^{k+1})\right\|^2\right]
     \nonumber
     \\ & +a_2(\alpha_k, n)\mathbb{E}\left[\left\|y^{k+1} -y^*(x^k)\right\|^2\right] +a_3(\alpha_k, n) \tilde{\sigma}^2_f,\nonumber
\end{align}
where the constants $a_1(\alpha_k),a_2(\alpha_k, n)$ and $a_3(\alpha_k, n)$ are given by 
\begin{align}
    & a_1(\alpha_k) :=  L_y^2 \alpha_k^2 +\frac{L_y\alpha_k}{4M_f}+\frac{L_{yx}\alpha_k^2}{2\eta},
    \nonumber
    \\ & a_2(\alpha_k, n):=  1+4M_fL_y\alpha_k +{\frac{\eta L_{yx}\tilde{D}^2_f\alpha_k^2}{2}},
    \nonumber
    \\ & a_3(\alpha_k, n):=  \frac{\alpha_k^2 L_y^2}{n} +\frac{L_{yx}\alpha_k^2}{2\eta n},
    \nonumber
\end{align}
for any $\eta > 0$.
\end{lemma}
\begin{proof}[\bf Proof of \Cref{le:error_FEDINN_post}]
Note that 
\begin{subequations}
\begin{align}
    \mathbb{E}\left[\left\|y^{k+1} -y^*(x^{k+1})\right\|^2\right]&=\mathbb{E}\left[\left\| y^{k+1} -y^*(x^k)\right\|^2\right]+\mathbb{E}\left[\left\|y^*(x^{k+1})-y^*(x^k)\right\|^2\right]\label{eq:error_FEDINN_post_main}
    \\
    &+2\mathbb{E}\left[\langle y^{k+1} -y^*(x^k),y^*(x^k)-y^*(x^{k+1})\rangle\right].\label{eq:error_FEDINN_post_main_two}
\end{align}
\end{subequations}
To bound the second term of~\cref{eq:error_FEDINN_post_main}, we have 
\begin{align}
    \mathbb{E}\left[\left\|y^*(x^{k+1})-y^*(x^k)\right\|^2\right] \leq & L_y^2\mathbb{E}\left[\left\|x^{k+1}-x^k\right\|^2\right]
    \nonumber
    \\ \leq & L_y^2 \mathbb{E}\left[\left\| \alpha_k \overline{\mathcal{H}}(x^k,y^{k+1})\right\|^2\right]+\frac{\alpha_k^2 L_y^2\tilde{\sigma}^2_f}{n}.
    \nonumber
\end{align}
To bound~\cref{{eq:error_FEDINN_post_main_two}}, we have 
\begin{subequations}
\begin{align}
    &\mathbb{E}\left[\langle y^{k+1} -y^*(x^k),y^*(x^k)-y^*(x^{k+1})\rangle\right]
    \nonumber\\
    &= -\mathbb{E}[
    \langle y^{k+1} -y^*(x^k), \nabla y^*(x^k)(x^{k+1}-x^k)\rangle]
    \label{eq:post_main_third_one}\\
    &-\mathbb{E} \left[\langle y^{k+1} -y^*(x^k),y^*(x^{k+1})-y^*(x^k)-\nabla y^*(x^k)(x^{k+1}-x^k)\rangle\right].\label{eq:post_main_third_two}
\end{align}
\end{subequations}
After plugging the updating step of $x^k$ into~\cref{{eq:post_main_third_one}}, we have  
\begin{align}\label{eq:post_main_third_term1}
    &-\mathbb{E}\left[\left\langle y^{k+1} -y^*(x^k), \nabla y^*(x^k)(x^{k+1}-x^k)\right\rangle\right]
    \nonumber
    \\
    &= -\mathbb{E}\left[\left\langle y^{k+1} -y^{*}(x^k),\alpha_k\nabla y^{*}(x^k) \overline{\mathcal{H}}(x^k,y^{k+1})\right\rangle\right]
    \nonumber
    \\ &\leq \mathbb{E}\left[\left\|y^{k+1}-y^*(x^k)\right\|\left\|{\alpha_k}\nabla y^*(x^k) \overline{\mathcal{H}}(x^k,y^{k+1})\right\|\right] 
    \nonumber
    \\ &\leq L_y \mathbb{E}\left[  
    \left\| y^{k+1} -y^*(x^k)\right\| \left\|{\alpha_k} \overline{\mathcal{H}}(x^k, y^{k+1})\right\| 
    \right]
    \nonumber 
    \\ &\stackrel{(a)}{\leq} 2\gamma \mathbb{E} \left[\left\|y^{k+1} -y^*(x^k)\right\|^2\right]+ \frac{L_y^2 \alpha_k^2}{8\gamma} \mathbb{E} \left[\left\| \overline{\mathcal{H}} (x^k, y^{k+1})\right\|^2\right]
    \nonumber 
    \\ &\leq 2M_f L_y\alpha_k \mathbb{E} \left[\left\|y^{k+1} -y^*(x^k)\right\|^2\right]+ {\frac{L_y \alpha_k}{8M_f} \mathbb{E} \left[\left\| \overline{\mathcal{H}} (x^k, y^{k+1})\right\|^2\right]},\nonumber 
\end{align}
where Young's inequality is applied in the inequality (a), and the last inequality comes from setting {$\gamma =M_f L_y \alpha_k$}.
\cref{{eq:post_main_third_two}} can be further bounded as follows, 
\begin{align}
    &-\mathbb{E} \left[\left\langle y^{k+1} -y^*(x^k),y^*(x^{k+1})-y^*(x^k)-\nabla y^*(x^k)(x^{k+1}-x^k)\right\rangle\right]
    \nonumber
    \\ &\leq  \mathbb{E} \left[ \left\| y^{k+1}-y^*(x^k)\right\|\left\|y^*(x^{k+1})-y^*(x^k)-\nabla y^*(x^k)(x^{k+1} -x^k)\right\|\right]
    \nonumber
    \\ &\leq  \frac{L_{yx}}{2}\mathbb{E} \left[\left\|y^{k+1} -y^*(x^k)\right\|\left\|x^{k+1}-x^k\right\|^2\right]
    \nonumber
    \\ &\leq  \frac{\eta L_{yx}}{4}\mathbb{E} \left[\left\|y^{k+1} -y^*(x^k)\right\|^2\left\|x^{k+1}-x^k\right\|^2\right]+\frac{L_{yx}}{4\eta} \mathbb{E}\left[\left\|x^{k+1} -x^k\right\|^2\right]
    \nonumber
    \\ &\leq  {\frac{\eta L_{yx}\tilde{D}^2_f \alpha_k^2}{4}\mathbb{E}\left[\left\|y^{k+1} -y^*(x^k)\right\|^2\right]}+\frac{L_{yx}\alpha_k^2}{4\eta}\mathbb{E}\left[\left\| \overline{\mathcal{H}}(x^k,y^{k+1})\right\|^2\right] +\frac{L_{yx}\alpha_k^2\tilde{\sigma}_f^2}{4\eta n},
    \nonumber
\end{align}
where~\Cref{le:bounded_hyper_gradient} is applied in the last inequality.

Combining and rearranging the above inequalities complete the proof. 
\end{proof}

\begin{proof}[\bf Proof of \Cref{thm:full_worker}]
Motivated by~\citealt{chen2021closing,tarzanagh2022fednest}, we define the following Lyapunov function
\begin{equation}
    W^k := f(x^k) +\frac{M_f}{L_y} \left\|y^k-y^*(x^k)\right\|^2.
    \nonumber
\end{equation}
The difference between the two Lyapunov functions is bounded as 
\begin{equation}
    W^{k+1} -W^k = f(x^{k+1})-f(x^k) +\frac{M_f}{L_y} (\left\|y^{k+1}-y^*(x^{k+1})\right\|^2-\left\|y^k-y^*(x^k)\right\|^2).
    \nonumber
\end{equation}
From \Cref{le:descent} and \Cref{le:error_FEDINN_post}, we obtain
\begin{subequations}
\begin{align}
    \mathbb{E}\left[W^{k+1}\right]-\mathbb{E}\left[W^k\right]\leq& {\alpha_k b^2} +\frac{\alpha_k^2 L_f \tilde{\sigma}_f^2}{n} +\alpha_k M_f^2 \mathbb{E}[\left\|y^{k+1}-y^{*}(x^k)\right\|^2]
\nonumber
\\ &+ (\alpha_k^2 L_f-\frac{\alpha_k}{2})\mathbb{E} \left[\left\| \overline{\mathcal{H}}(x^k,y^{k+1})\right\|^2\right] -\frac{\alpha_k}{2} \mathbb{E}\left[\left\|\nabla f(x^k)\right\|^2\right]
\nonumber 
\\ &+ \frac{a_1(\alpha_k)M_f}{L_y} \mathbb{E}\left[\left\|\overline{\mathcal{H}}(x^k,y^{k+1})\right\|^2\right]
     \nonumber
     \\ & +\frac{a_2(\alpha_k, n)M_f}{L_y}\mathbb{E}\left[\left\|y^{k+1} -y^*(x^k)\right\|^2\right] +\frac{a_3(\alpha_k, n)M_f\tilde{\sigma}^2_f}{L_y}
     \nonumber
     \\ & -\frac{M_f}{L_y} \mathbb{E} \left[\left\|y^k-y^*(x^k)\right\|\right]
     \nonumber
     \\ = & { \alpha_k b^2} +(\frac{\alpha_k^2 L_f }{n}+\frac{a_3(\alpha_k,n)M_f}{L_y})\tilde{\sigma}_f^2-\frac{\alpha_k}{2}\mathbb{E}\left [\left\|\nabla f(x^k)\right\|^2\right]
     \nonumber
     \\ & + (\alpha_k^2 L_f-\frac{\alpha_k}{2}+\frac{a_1(\alpha_k)M_f}{L_y})\mathbb{E} \left[\left\|  \overline{\mathcal{H}}(x^k,y^{k+1})\right\|^2\right]\label{eq:main_thm_enforce_negative}
     \\ &+ (\alpha_k M_f^2+ \frac{a_2(\alpha_k, n)M_f}{L_y})\mathbb{E} \left[\left\|y^{k+1}-y^*(x^k)\right\|^2\right]-\frac{M_f}{L_y}\mathbb{E}\left[\left\|y^k-y^*(x^k)\right\|\right].\label{eq:main_thm_enforce_negative2}
\end{align}
\end{subequations}
Note that \cref{eq:main_thm_enforce_negative}$\leq 0$ if 
\begin{equation}\label{thm:alpha_1}
    \alpha_k \leq \hat{\alpha}_1:=\frac{1}{2L_f +4M_f L_y +\frac{2M_f L_{yx}}{L_y \eta}}
\end{equation}
We enforce $\alpha_k \leq \hat{\alpha}_1$ in the following context. 

Based on \cref{le:error_FEDINN_pre}, \cref{eq:main_thm_enforce_negative2} can be further bounded as 
\begin{align}
    \cref{eq:main_thm_enforce_negative2}\leq & 4(\alpha_k M_f^2+ \frac{a_2(\alpha_k, n)M_f}{L_y})\frac{T\beta_{k,t}^2(\sigma_{g,1}^2+\sigma_g^2)}{n {S}}
    \nonumber
    \\ & + \frac{M_f}{L_y} \left( \left(M_f L_y\alpha_k+a_2(\alpha_k,n)\right)\left(1-\beta_{k,t}\mu_g\right)^T-1\right) \mathbb{E}\left[\left\|y^k-y^*(x^k)\right\|\right].\label{eq:full_thm_enforce_negative3}
\end{align}
If $\beta_{k,t} \leq \frac{1}{\mu_g}$, \cref{eq:full_thm_enforce_negative3} is nonpositive if 
\begin{align}
    &(1+5M_fL_y\alpha_k+\frac{\eta L_{yx}\tilde{D}_f^2\alpha_k^2}{2})(1-\beta_{k,t}\mu_g)^T-1 \leq 0
    \nonumber
     \\ & \Leftarrow 5M_fL_y\alpha_k+\frac{\eta L_{yx}\tilde{D}_f^2\alpha_k^2}{2}\leq T\beta_{k,t}\mu_g
    \nonumber
    \\ & \Leftarrow
    \beta_{k,t} \geq \left(\frac{5M_f L_y}{\mu_g}+\frac{\eta L_{yx}\tilde{D}^2_f\alpha_k}{2\mu_g}\right)\frac{\alpha_k}{T}\nonumber
\end{align}
For simplicity, we remove the subscript $t$ from $\beta_{k,t}$ and enforce
\begin{equation}\label{eq:main_enforce_beta_1}
    \beta_{k} = \bar{\beta}\frac{\alpha_k}{T}\nonumber
\end{equation}
where 
\begin{equation}
    \bar{\beta}:=\frac{5M_f L_y}{\mu_g}+\frac{\eta L_{yx}\tilde{D}^2_f\hat{\alpha}_1}{2\mu_g},\label{thm:beta_bar}
\end{equation}
which will imply another requirement on $\alpha_k$ since $\beta_{k}$ should be less than $\frac{1}{2l_{g,1}}$ as a condition of \Cref{le:error_FEDINN_pre}, i.e.,
\begin{equation}
    \alpha_k \leq \hat{\alpha}_2:= \frac{T}{2 l_{g,1}\bar{\beta}}.\label{thm:alpha_2}
\end{equation}
After rearranging, we obtain
\begin{align}
    \mathbb{E}\left[W^{k+1}\right]-\mathbb{E}\left[W^k\right]\leq & \alpha_k b^2 +\left(\frac{\alpha_k^2 L_f}{n}+\frac{a_3(\alpha_k,n)M_f}{L_y}\right)\tilde{\sigma}_f^2-\frac{\alpha_k}{2}\mathbb{E}\left [\left\|\nabla f(x^k)\right\|^2\right]
     \nonumber
     \\ &+4\left(\alpha_k M_f^2+ \frac{a_2(\alpha_k, n)M_f}{L_y}\right)\frac{T\beta_k^2(\sigma_{g,1}^2+\sigma_g^2)}{n{S}}.\label{eq:main_thm_enforce_result1}\nonumber
\end{align}
Then telescoping gives
\begin{align}
    \frac{1}{K} \sum_{k=0}^{K-1} \mathbb{E} \left[\left\|\nabla f(x^k)\right\|^2\right]\leq & \frac{2}{\sum_{k=0}^{K-1}\alpha_k}(\triangle_w)+{2b^2} +\frac{2}{\sum_{k=0}^{K-1}\alpha_k}\sum_{k=0}^{K-1}\left(\frac{\alpha_k^2 L_f}{n}+\frac{a_3(\alpha_k,n)M_f}{L_y}\right)\tilde{\sigma}_f^2
    \nonumber
    \\ & + \frac{8}{\sum_{k=0}^{K-1}\alpha_k} \sum_{k=0}^{K-1}\left(\alpha_k M_f^2+ \frac{a_2(\alpha_k, n)M_f}{L_y}\right)\frac{T\beta_k^2(\sigma_{g,1}^2+\sigma_g^2)}{nS},\nonumber
\end{align}
where $\triangle_w:=W^0-\mathbb{E}\left[W^K\right]$. We enforce $\alpha_k \leq \sqrt{\frac{n}{K}} \hat{\alpha}_3$ for some positive constant $\hat{\alpha}_3$, which implies 
\begin{subequations}\label{eq:main_thm_T1}
\begin{align}
    \frac{2}{\sum_{k=0}^{K-1}\alpha_k}(\triangle_w) &= \mathcal{O}\left(\frac{1}{\min(\hat{\alpha}_1,\hat{\alpha}_2)K}+{\frac{1}{\hat{\alpha}_3\sqrt{nK}}}\right)
    \\ \frac{2}{\sum_{k=0}^{K-1}\alpha_k}\cdot\sum_{k=0}^{K-1}(\frac{\alpha_k^2 L_f}{n}+\frac{a_3(\alpha_k,n)M_f}{L_y})\tilde{\sigma}_f^2 &=\mathcal{O} \left(\frac{2}{\sum_{k=0}^{K-1}\alpha_k}\cdot \sum_{k=0}^{K-1}\frac{\alpha_k^2}{n}\right)
  = \mathcal{O}\left(\frac{\hat{\alpha}_3}{\sqrt{nK}}\right)
    \\\frac{8}{\sum_{k=0}^{K-1}\alpha_k}\cdot \sum_{k=0}^{K-1}\Big(\alpha_k M_f^2+ \frac{a_2(\alpha_k, n)M_f}{L_y}\Big)\frac{T\beta_k^2(\sigma_{g,1}^2+\sigma_g^2)}{{nS}}&=\mathcal{O}\left(\frac{4}{\sum_{k=0}^{K-1}\alpha_k}\sum_{k=0}^{K-1}\frac{\alpha_k^2}{STn}+\frac{\alpha_k^3}{STn}+\frac{\alpha_k^4}{STn}\right)
    \nonumber\\
    &=\mathcal{O}\left(\frac{\hat{\alpha}_3}{ST\sqrt{nK}}+\frac{\hat{\alpha}_3^2}{STK}+\frac{{\sqrt{n}}\hat{\alpha}_3^3}{STK^{3/2}}\right)
\end{align}
\end{subequations}
Therefore, we obtain 
\begin{equation}
     \frac{1}{K} \sum_{k=0}^{K-1} \mathbb{E} \left[\left\|\nabla f(x^k)\right\|^2\right] =\mathcal{O}\left(\frac{\hat{\alpha}_3+\hat{\alpha}_3^{-1}}{\sqrt{nK}}+\frac{1}{\min(\hat{\alpha}_1,\hat{\alpha}_2)K}+{b^2}\right).\nonumber
\end{equation}
Then, the proof is complete. 
\end{proof}

\begin{proof}[\bf Proof of \Cref{corollary}]
    Enforcing $\eta=\frac{M_f}{L_y}$ in~\cref{thm:alpha_1},~\cref{thm:alpha_2} and~\cref{thm:beta_bar}, yields $\hat{\alpha}_1=\mathcal{O}(\kappa_g^{-3})$, $\hat{\alpha}_2=\mathcal{O}(T\kappa_g^{-3})$ and $\bar{\beta}=\mathcal{O}(\kappa^4)$.  
    Expanding~\cref{eq:main_thm_T1}, we have 
    \begin{align}
    &\frac{2}{\sum_{k=0}^{K-1}\alpha_k}(\triangle_w)= \mathcal{O}(\frac{1}{\min(\hat{\alpha}_1,\hat{\alpha}_2)K}+{\frac{1}{\hat{\alpha}_3\sqrt{nK}}})
    \nonumber
    \\
    &\frac{2}{\sum_{k=0}^{K-1}\alpha_k}\cdot\sum_{k=0}^{K-1}(\frac{\alpha_k^2 L_f}{n}+\frac{a_3(\alpha_k,n)M_f}{L_y})\tilde{\sigma}_f^2\nonumber\\
    &=\mathcal{O}\left(\frac{2}{\sum_{k=0}^{K-1}\alpha_k}\cdot \sum_{k=0}^{K-1}\frac{\kappa_g^3\alpha_k^2}{n}+\frac{\kappa_g^5\alpha_k^2}{n}\right)= \mathcal{O}\left(\frac{\kappa_g^5\hat{\alpha}_3}{\sqrt{nK}}\right)
    \nonumber
    \\
    &\frac{8}{\sum_{k=0}^{K-1}\alpha_k}\cdot \sum_{k=0}^{K-1}(\alpha_k M_f^2+ \frac{a_2(\alpha_k, n)M_f}{L_y})\frac{T\beta_k^2(\sigma_{g,1}^2+\sigma_g^2)}{{nS}}\nonumber\\
    &=\mathcal{O}\left(\frac{4}{\sum_{k=0}^{K-1}\alpha_k}\cdot\sum_{k=0}^{K-1}\frac{\eta\bar{\beta}^2}{nST}\alpha_k^2+\left(\frac{M_f^2\bar{\beta}^2}{nST}+\frac{\eta M_f L_y \bar{\beta}^2}{nST}\right)\alpha_k^3+\frac{\eta^2 \bar{\beta}^2 L_{yz}\tilde{D}_f^2}{nST}\alpha_k^4\right)
    \nonumber\\
    &=\mathcal{O}\left(\frac{4}{\sum_{k=0}^{K-1}\alpha_k}\cdot\sum_{k=0}^{K-1}\frac{\kappa_g^9}{nST}\alpha_k^2+\frac{\kappa_g^{12}}{nST}\alpha_k^3+\frac{\kappa_g^{15}
    }{nST}\alpha_k^4\right)\nonumber
    \\&=\mathcal{O}\left(\frac{\kappa_g^9}{ST\sqrt{nK}}\hat{\alpha}_3+\frac{\kappa_g^{12}}{STK}\hat{\alpha}_3^2+\frac{\kappa_g^{15}\sqrt{n}}{STK^{3/2}}\hat{\alpha}_3^3
    \right)\nonumber
\end{align}
After enforcing $ST=\Omega(\kappa^4)$, $\hat{\alpha}_3=\mathcal{O}\left(\kappa_g^{-5/2}\right)$, and $N=\Omega(\kappa_g \log K)$, which implies $b=\frac{1}{K^{1/4}}$, we have 
\begin{align}
    \frac{1}{K} \sum_{k=0}^{K-1} \mathbb{E} \left[\left\|\nabla f(x^k)\right\|^2\right] =\mathcal{O}\left(\frac{\kappa_g^{5/2}}{\sqrt{nK}}+\frac{\kappa_g^3}{K}+\frac{\kappa_g^{7/2}\sqrt{n}}{K^{3/2}}\right).\nonumber
\end{align}
Then, the proof is complete. 
\end{proof}


\end{document}